\documentclass[10pt]{article}

\usepackage{mathtools}

\usepackage{amsmath}
\usepackage{amssymb}
\usepackage{amsthm}

\usepackage{authblk}

\usepackage{comment}
\usepackage{graphicx}
\usepackage{subcaption}

\newtheorem{theorem}{Theorem}
\newtheorem{lemma}{Lemma}

\usepackage[numbers]{natbib}

\begin{document}

\title{Bayesian Best-Arm Identification for Selecting Influenza Mitigation Strategies}

\author[1,2]{Pieter Libin}
\author[1]{Timothy Verstraeten}
\author[1]{Diederik M. Roijers}
\author[1]{Jelena Grujic}
\author[2]{Kristof Theys}
\author[2]{Philippe Lemey}
\author[1]{Ann Now\'{e}}
\affil[1]{Artificial Intelligence Lab, Department of computer science, Vrije Universiteit Brussel, Brussels, Belgium}
\affil[2]{KU Leuven – University of Leuven, Department of Microbiology and Immunology, Rega Institute for Medical Research, Clinical and Epidemiological Virology, Leuven, Belgium}

\maketitle

 \begin{abstract}
Pandemic influenza has the epidemic potential to kill millions of people. While various preventive measures exist (i.a., vaccination and school closures), deciding on strategies that lead to their most effective and efficient use remains challenging. To this end, individual-based epidemiological models are essential to assist decision makers in determining the best strategy to curb epidemic spread. However, individual-based models are computationally intensive and it is therefore pivotal to identify the optimal strategy using a minimal amount of model evaluations. Additionally, as epidemiological modeling experiments need to be planned, a computational budget needs to be specified a priori. Consequently, we present a new sampling technique to optimize the evaluation of preventive strategies using fixed budget best-arm identification algorithms. We use epidemiological modeling theory to derive knowledge about the reward distribution which we exploit using Bayesian best-arm identification algorithms (i.e., Top-two Thompson sampling and BayesGap). We evaluate these algorithms in a realistic experimental setting and demonstrate that it is possible to identify the optimal strategy using only a limited number of model evaluations, i.e., 2-to-3 times faster compared to the uniform sampling method, the predominant technique used for epidemiological decision making in the literature. Finally, we contribute and evaluate a statistic for Top-two Thompson sampling to inform the decision makers about the confidence of an arm recommendation.

 \end{abstract}

 \section{Introduction}
The influenza virus is responsible for the deaths of half of a million people each year.
In addition,
seasonal influenza epidemics cause a significant economic burden. While transmission is primarily local, a newly emerging variant may spread to pandemic proportions in a fully susceptible host population \cite{Paules2017}. Pandemic influenza occurs less frequently than seasonal influenza but the outcome with respect to morbidity and mortality can be much more severe, potentially killing millions of people worldwide \cite{Paules2017}. Therefore, it is essential to study mitigation strategies to control influenza pandemics.

For influenza, different preventive measures exist: i.a., vaccination, social measures (e.g., school closures and travel restrictions) and antiviral drugs. However, the efficiency of strategies greatly depends on the availability of preventive compounds, as well as on the characteristics of the targeted epidemic. Furthermore, governments typically have limited resources to implement such measures. Therefore, it remains challenging to formulate public health strategies that make effective and efficient use of these preventive measures
within the existing resource constraints.

Epidemiological models (i.e., compartment models and individual-based models) are essential to study the effects of preventive measures \textit{in silico} \cite{Basta2009,Germann2006}.
While individual-based models are usually associated with a greater model complexity and computational cost than compartment models, they allow for a more accurate evaluation of preventive strategies \cite{Eubank2006}.
To capitalize on these advantages and make it feasible to employ individual-based models, it is essential to use the available computational resources as efficiently as possible.

In the literature, a set of possible preventive strategies is typically evaluated by simulating each of the strategies an equal number of times \cite{Fumanelli2016,ferguson2005strategies,Chao2012}.
However, this approach is inefficient to identify the optimal preventive strategy, as a large proportion of computational resources will be used to explore sub-optimal strategies.
Furthermore, a consensus on the required number of model evaluations per strategy is currently lacking \cite{Willem2014} and we show that this number depends on the \emph{hardness} of the evaluation problem.
Additionally, we recognize that epidemiological modeling experiments need to be planned and that a computational budget needs to be specified a priori.
Therefore, we present a novel approach where we formulate the evaluation of preventive strategies as a \emph{best-arm identification} problem using a \emph{fixed budget} of model evaluations.

As running an individual-based model is computationally intensive (i.e., minutes to hours, depending on the complexity of the model), minimizing the number of required model evaluations reduces the total time required to evaluate a given set of preventive strategies. This renders the use of individual-based models attainable in studies where it would otherwise not be computationally feasible. Additionally, reducing the number of model evaluations will free up computational resources in studies that already use individual-based models, capacitating researchers to explore a larget set of model scenarios. This is important, as considering a wider range of scenarios increases the confidence about the overall utility of preventive strategies \cite{wu2006reducing}.

In this paper, we contribute a novel technique to evaluate preventive strategies as a fixed budget best-arm identification problem.
We employ epidemiological modeling theory to derive assumptions about the reward distribution and exploit this knowledge using Bayesian algorithms.
This new technique enables decision makers to obtain recommendations in a reduced number of model evaluations.
We evaluate the technique in an experimental setting, where we aim to find the best vaccine allocation strategy in a realistic simulation environment that models an influenza pandemic on a large social network.
Finally, we contribute and evaluate a statistic to inform the decision makers about the confidence of a particular recommendation.

\section{Background}
\label{sec:background}

\subsection{Pandemic influenza and vaccine production}
\label{subsec:bg:pandemicinfluenza}
The primary preventive strategy to mitigate seasonal influenza is to produce vaccine prior to the epidemic, anticipating the virus strains that are expected to circulate.
This vaccine pool is used to inoculate the population before the start of the epidemic. 

While it is possible to stockpile vaccines to prepare for seasonal influenza, this is not the case for influenza pandemics, as the vaccine should be specifically tailored to the virus that is the source of the pandemic.
Therefore, before an appropriate vaccine can be produced, the responsible virus needs to be identified.
Hence, vaccines will be available only in limited supply at the beginning of the pandemic \cite{who2004}.
In addition, production problems can result in vaccine shortages \cite{enserink2004}.
When the number of vaccine doses is limited, it is imperative to identify an optimal vaccine allocation strategy \cite{Medlock2009}.

\subsection{Modeling influenza}
There is a long tradition of using individual-based models to study influenza epidemics \cite{Basta2009,Germann2006,Fumanelli2016}, as they allow for a more accurate evaluation of preventive strategies. A state-of-the-art individual-based model that has been the driver for many high impact research efforts \cite{Basta2009,Germann2006,Halloran2002}, is FluTE \cite{chao2010}.

FluTE implements a contact model where the population is divided into communities of households \cite{chao2010}. The population is organized in a hierarchy of social mixing groups where the contact intensity is inversely proportional with the size of the group (e.g., closer contact between members of a household than between colleagues). Additionally, FluTE implements an individual disease progression model that associates different disease stages with different levels of infectiousness. FluTE supports the evaluation of preventive strategies through the simulation of therapeutic interventions (i.e., vaccines, antiviral compounds) and non-therapeutic interventions (i.e., school closure, case isolation, household quarantine).

\subsection{Bandits and best-arm identification}
\label{subsec:bg:bestarm}
The \emph{multi-armed bandit game} \cite{audibert2010best} involves a $K$-armed bandit (i.e., a slot machine with $K$ levers), where each arm $A_k$ returns a reward $r_k$ when it is pulled (i.e., $r_k$ represents a sample from $A_k$'s reward distribution).
A common use of the bandit game is to pull a sequence of arms such that the cumulative regret is minimized \cite{Herbert1952}. To fulfill this goal, the player needs to carefully balance between exploitation and exploration.

In this paper, the objective is to recommend the best arm $A^*$ (i.e., the arm with the highest average reward $\mu^*$), after a fixed number of arm pulls. This is referred to as the fixed budget best-arm identification problem \cite{audibert2010best}, an instance of the pure-exploration problem \cite{bubeck2009pure}. For a given budget $T$, the objective is to minimize the \emph{simple regret} $\mu^* - \mu_J$, where $\mu_J$ is the average reward of the recommended arm $A_J$, at time T \cite{bubeck2011pure}. Simple regret is inversely proportional to the probability of recommending the correct arm $A^*$ \cite{kaufmann2016complexity}.

\section{Related work}
As we established that a computational budget needs to be specified a priori, our problem setting matches the fixed budget best-arm identification setting. This differs from settings that attempt to identify the best arm with a predefined confidence: i.e., racing strategies \cite{even2006action}, strategies that exploit the confidence bound of the arms' means \cite{kaufmann2013information} and more recently fixed confidence best-arm identification algorithms \cite{garivier2016optimal}.

We selected Bayesian fixed budget best-arm identification algorithms, as we aim to incorporate prior knowledge about the arms' reward distributions and use the arms' posteriors to define a statistic to support policy makers with their decisions.
We refer to \cite{kaufmann2016complexity,hoffman2014correlation}, for a broader overview of the state of the art with respect to (Bayesian) best-arm identification algorithms.

Best-arm identification algorithms have been used in a large set of application domains: i.a., evaluation of response surfaces, the initialization of hyper-parameters and traffic congestion.

While other algorithms exist to rank or select bandit arms, e.g. \cite{powell2012optimal}, best-arm identification is best approached using adaptive sampling methods \cite{jennison1982asymptotically}, as the ones we study in this paper.

In preliminary work, we explored the potential of multi-armed bandits to evaluate prevention strategies in a regret minimization setting, using default strategies (i.e., $\epsilon$-greedy and UCB1). We presented this work at the the 'Adaptive Learning Agents' workshop hosted by the AAMAS conference \cite{Libin2017}. This setting is however inadequate to evaluate prevention strategies \emph{in silico}, as minimizing cumulative regret is sub-optimal to identify the best arm. Additionally, in this workshop paper, the experiments considered a small and less realistic population, and only analyzed a limited range of $R_0$ values that is not representative for influenza pandemics.

\section{Methods}
\label{sec:methods}
We formulate the evaluation of preventive strategies as a multi-armed bandit game with the aim of identifying the best arm using a fixed budget of model evaluations. The presented method is generic with respect to the type of epidemic that is modeled (i.e., pathogen, contact network, preventive strategies). The method is evaluated in the context of pandemic influenza in the next section.

\subsection{Evaluating preventive strategies with bandits}
 A \emph{stochastic epidemiological model} $E$ is defined in terms of a model configuration $c \in \mathcal{C}$ and can be used to evaluate a preventive strategy $p \in \mathcal{P}$.
The result of a model evaluation is referred to as the \emph{model outcome} (e.g., prevalence, proportion of symptomatic individuals, morbidity, mortality, societal cost).
Evaluating the model $E$ thus results in a sample of the model's \emph{outcome distribution}:
  \begin{equation}
   \text{outcome} \sim E(c, p) \text{, where } c \in \mathcal{C} \text{ and } p \in \mathcal{P}
  \end{equation}

Our objective is to find the optimal preventive strategy (i.e., the strategy that minimizes the expected outcome) from a set of alternative strategies $\{p_1,...,p_K\} \subset \mathcal{P}$ for a particular configuration $c_0 \in \mathcal{C}$ of a stochastic epidemiological model, where $c_0$ corresponds to the studied epidemic.
To this end, we consider a  multi-armed bandit with $K=|\{p_1,...,p_{K}\}|$ arms. Pulling arm $p_k$ corresponds to evaluating $p_k$ by running a simulation in the epidemiological model $E(c_0, p_k)$. The bandit thus has preventive strategies as arms with reward distributions corresponding to the outcome distribution of a stochastic epidemiological model $E(c_0, p_k)$.
While the parameters of the reward distribution are known (i.e., the parameters of the epidemiological model), it is intractable to determine the optimal reward analytically. Hence, we must learn about the outcome distribution via interaction with the epidemiological model.
In this work, we consider prevention strategies of equal financial cost, which is a realistic assumption, as governments typically operate within budget constraints. 

\subsection{Outcome distribution}
\label{subsec:methods:outcome_distr}
As previously defined, the reward distribution associated with a bandit's arm corresponds to the outcome distribution of the epidemiological model that is evaluated when pulling that arm.
Therefore, employing insights from epidemiological modeling theory allows us to specify prior knowledge about the reward distribution.

It is well known that a disease outbreak has two possible outcomes: either it is able to spread beyond a local context and becomes a fully established epidemic or it fades out \cite{Watts2005}.
Most stochastic epidemiological models reflect this reality and hence its epidemic size distribution is bimodal \cite{Watts2005}.
When evaluating preventive strategies, the objective is to determine the preventive strategy that is most suitable to mitigate an established epidemic.
As in practice we can only observe and act on established epidemics, epidemics that faded out in simulation would bias this evaluation.
Consequently, it is necessary to focus on the mode of the distribution that is associated with the established epidemic.
Therefore we censor (i.e., discard) the epidemic sizes that correspond to the faded epidemic.
The size distribution that remains (i.e., the one that corresponds with the established epidemic) is approximately Gaussian \cite{Britton2010}.

In this study, we consider a scaled epidemic size distribution, i.e., the proportion of symptomatic infections.
Hence we can assume bimodality of the full size distribution and an approximately Gaussian size distribution of the established epidemic.
We verified experimentally that these assumptions hold for all the reward distributions that we observed in our experiments (see Section~\ref{sec:experiments}).

To censor the size distribution, we use a threshold that represents the number of infectious individuals that are required to ensure an outbreak will only fade out with a low probability.

\subsection{Epidemic fade-out threshold}
\label{subsec:methods:threshold}
For heterogeneous host populations (i.e., a population with a significant variance among individual transmission rates, as is the case for influenza epidemics \cite{dorigatti2012,fraser2011}), the number of secondary infections can be accurately modeled using a negative binomial \emph{offspring distribution} $\text{NB}(R_0,\gamma)$ \cite{Lloyd2005}, where $R_0$ is the basic reproductive number (i.e., the number of infections that is, by average, generated by one single infection) and $\gamma$ is a dispersion parameter that specifies the extent of heterogeneity.
The probability of epidemic extinction $p_{\text{ext}}$ can be computed by solving $g(s)=s$, where $g(s)$ is the probability generating function (pgf) of the offspring distribution \cite{Lloyd2005}.
For an epidemic where individuals are targeted with preventive measures (i.e., vaccination in our case), we obtain the following pgf
\begin{equation}
  g(s)=pop_c+(1-pop_c)\big(1+\frac{R_0}{\gamma}(1-s) \big)^{-\gamma}
\end{equation}
where $pop_c$ signifies the random proportion of controlled individuals \cite{Lloyd2005}.
From $p_{\text{ext}}$ we can compute a threshold $T_0$ to limit the probability of extinction to a cutoff $\ell$ \cite{Hartfield2013}.

\subsection{Best-arm identification with a fixed budget}
\label{subsec:methods:bestarmid}
Our objective is to identify the best preventive strategy (i.e., the strategy that minimizes the expected outcome) out of a set of preventive strategies, for a particular configuration $c_0 \in \mathcal{C}$ using a fixed budget $T$ of model evaluations.
To find the best prevention strategy, it suffices to focus on the mean of the outcome distribution, as it is approximately Gaussian with an unknown yet small variance \cite{Britton2010}, as we confirm in our experiments (see Figure~\ref{fig:violin}).

Successive Rejects was the first algorithm to solve the best-arm identification in a fixed budget setting \cite{audibert2010best}. For a $K$-armed bandit, Successive Rejects operates in $(K-1)$ phases. At the end of each phase, the arm with the lowest average reward is discarded. Thus, at the end of phase $(K-1)$ only one arm survives, and this arm is recommended.

Successive Rejects serves as a useful baseline, however, it has no support to incorporate any prior knowledge.
Bayesian best-arm identification algorithms on the other hand, are able to take into account such knowledge by defining an appropriate prior and posterior on the arms' reward distribution.
As we will show, such prior knowledge can increase the best-arm identification accuracy. Additionally, at the time an arm is recommended, the posteriors contain valuable information that can be used to formulate a variety of statistics helpful to assist decision makers.
We consider two state-of-the-art Bayesian algorithms: BayesGap \cite{hoffman2014correlation} and Top-two Thompson sampling \cite{russo2016simple}.
For Top-two Thompson sampling, we derive a statistic based on the posteriors to inform the decision makers about the confidence of an arm recommendation: the probability of success.

As we established in the previous section, each arm of our bandit has a reward distribution that is approximately Gaussian with unknown mean and variance. For the purpose of genericity, we assume an uninformative Jeffreys prior $(\sigma_k)^{-3}$ on $(\mu_k, \sigma^2_k)$, which leads to the following posterior on $\mu_k$ at the $n_k^{\text{th}}$ pull \cite{Honda2014}:
\begin{equation}
  \label{eq:posterior}
\sqrt{\frac{n_k^2}{S_{k,n_k}}}(\mu_k - \overline{x}_{k,n_k}) \hspace{1mm} | \hspace{1mm} \overline{x}_{k,n_k},S_{k,n_k} \sim \mathcal{T}_{n_k}
\end{equation}
where $\overline{x}_{k,n_k}$ is the reward mean, $S_{k,n_k}$ is the sum of squares\begin{equation}
S_{k,n_k} = \sum_{m=1}^{n_k}(r_{k,m} - \overline{x}_{k,n_k})^2
\end{equation}
and $\mathcal{T}_{n_k}$ is the standard student t-distribution with $n_k$ degrees of freedom.

BayesGap is a gap-based Bayesian algorithm \cite{hoffman2014correlation}.
The algorithm requires that for each arm $A_k$, a high-probability upper bound $U_k(t)$ and lower bound $L_k(t)$ is defined on the posterior of $\mu_k$ at each time step $t$.
Using these bounds, the gap quantity
\begin{equation}
  B_k(t)=\max_{l \neq k}U_{l}(t) - L_k(t)
\end{equation}
is defined for each arm $A_k$. $B_k(t)$ represents an upper bound on the simple regret (as defined in Section~\ref{subsec:bg:bestarm}).
At each step $t$ of the algorithm, the arm $J(t)$ that minimizes the gap quantity $B_k(t)$ is compared to the arm $j(t)$ that maximizes the upper bound $U_k(t)$.
From $J(t)$ and $j(t)$, the arm with the highest confidence diameter $U_k(t)-L_k(t)$ is pulled.
The reward that results from this pull is observed and used to update $A_k$'s posterior.
When the budget is consumed, the arm
\begin{equation}
  J(\operatornamewithlimits{argmin}\limits_{t \leq T} B_{J(t)}(t))
\end{equation}
is recommended. This is the arm that minimizes the simple regret bound over all times $t \leq T$.

In order to use BayesGap to evaluate preventive strategies, we contribute problem-specific bounds. Given our posteriors (Equation ~\ref{eq:posterior}), we define
\begin{equation}
\begin{split}
  U_k(t)&=\hat{\mu}_k(t) + \beta \hat{\sigma}_k(t)\\
  L_k(t)&=\hat{\mu}_k(t) - \beta \hat{\sigma}_k(t)
\end{split}
\end{equation}
where $\hat{\mu}_k(t)$ and $\hat{\sigma}_k(t)$ are the respective mean and standard deviation of the posterior of arm $A_k$ at time step $t$, and $\beta$ is the exploration coefficient.

The amount of exploration that is feasible given a particular bandit game, is proportional to the available budget, and inversely proportional to the game's complexity \cite{hoffman2014correlation}.
This complexity can be modeled taking into account the game's hardness \cite{audibert2010best} and the variance of the rewards.
We use the hardness quantity defined in \cite{hoffman2014correlation}:
\begin{equation}
  H_{\epsilon} = \sum_k{H_{k,\epsilon}^{-2}}
\end{equation}
with arm-dependent hardness
\begin{equation}
  H_{k,\epsilon} = \max(\frac{1}{2}(\Delta_k + \epsilon), \epsilon)
  \text{, where }
  \Delta_k = \max_{l \neq k}(\mu_l) - \mu_k
\end{equation}

Considering the budget $T$, hardness $H_{\epsilon}$ and a generalized reward variance $\sigma_G^2$ over all arms, we define
\begin{equation}
  \beta=\sqrt{\frac{T - 3K}{4 H_{\epsilon} \sigma_G^2}}
\end{equation}
Theorem 1 in the Supplementary Information (Section 2) formally proves that using these bounds results in a probability of simple regret that asymptotically reaches the exponential lower bound of \cite{hoffman2014correlation}.

As both $H_{\epsilon}$ and $\sigma_G^2$ are unknown, in order to compute $\beta$, these quantities need to be estimated.
Firstly, we estimate $H_{\epsilon}$'s upper bound $\hat{H}_{\epsilon}$ by estimating $\Delta_k$ as follows
\begin{equation}
  \hat{\Delta}_k = \max_{1 \leq l < K; l \neq k}{(\hat{\mu}_l(t) + 3\hat{\sigma}_l(t)\big)} - \big(\hat{\mu}_k(t) - 3\hat{\sigma}_k(t))
\end{equation}
as in \cite{hoffman2014correlation},
where $\hat{\mu}_k(t)$ and $\hat{\sigma}_k(t)$ are the respective mean and standard deviation of the posterior of arm $A_k$ at time step $t$.
Secondly, for $\sigma_G^2$ we need a measure of variance that is representative for the reward distribution of all arms.
To this end, when the arms are initialized, we observe their sample variance $s_k^2$, and compute their average $\bar{s}_G^2$.

As our bounds depend on the standard deviation $\hat{\sigma}_k(t)$ of the t-distributed posterior, each arm's posterior needs to be initialized 3 times (i.e., by pulling the arm) to ensure that $\hat{\sigma}_k(t)$ is defined, this initialization also ensures proper posteriors \cite{Honda2014}.

Top-two Thompson sampling is a reformulation of the Thompson sampling algorithm, such that it can be used in a pure-exploration context \cite{russo2016simple}.
Thompson sampling operates directly on the arms' posterior of the reward distribution's mean $\mu_k$.
At each time step, Thompson sampling obtains one sample for each arm's posterior.
The arm with the highest sample is pulled, and its reward is subsequently used to update that arm's posterior.
While this approach has been proven highly successful to minimize cumulative regret \cite{Chapelle2011,Honda2014}, as it balances the exploration-exploitation trade-off, it is sub-optimal to identify the best arm \cite{bubeck2009pure}.
To adapt Thompson sampling to minimize simple regret, Top-two Thompson sampling  increases the amount of exploration.
To this end, an exploration probability $\omega$ needs to be specified.
At each time step, one sample is obtained for each arm's posterior.
The arm $A_{\text{top}}$ with the highest sample is only pulled with probability $\omega$.
With probability $1-\omega$ we repeat sampling from the posteriors until we find an arm $A_{\text{top-2}}$ that has the highest posterior sample and where $A_{\text{top}} \neq A_{\text{top-2}}$.
When the arm $A_{\text{top-2}}$ is found, it is pulled and the observed reward is used to update the posterior of the pulled arm.
When the available budget is consumed, the arm with the highest average reward is recommended.

As Top-two Thompson sampling only requires samples from the arms' posteriors, we can use the t-distributed posteriors from Equation~\ref{eq:posterior} as is. To avoid improper posteriors, each arm needs to be initialized 2 times \cite{Honda2014}.

As specified in the previous subsection, the reward distribution is censored. We observe each reward, but only consider it to update the arm's value
when it exceeds the threshold $T_0$ (i.e., when we receive a sample from the mode of the epidemic that represents the established epidemic).

\subsection{Probability of success}
The probability that an arm recommendation is correct presents a useful confidence statistic to support policy makers with their decisions.
As Top-two Thompson sampling recommends the arm with the highest average reward, and we assume that the arm's reward distributions are independent, the probability of success is
\begin{equation}
  P(\mu_{_{J}} = \max_{1 \leq k \leq K}{\mu_{_k}}) = \int_{x \in \mathbb{R}} \big[\prod_{k \neq J}^{K}F_{\mu_k}(x)\big]f_{\mu_J}(x)dx
\end{equation}
where $\mu_{_{J}}$ is the random variable that represents the mean of the recommended arm's reward distribution, $f_{\mu_J}$ is the recommended arm's posterior probability density function and $F_{\mu_k}$ is the other arms' cumulative density function. As this integral cannot be computed analytically, we estimate it using Gaussian quadrature.

It is important to note that, while aiming for generality, we made some conservative assumptions: the reward distributions are approximated as Gaussian and the uninformative Jeffreys prior is used. These assumptions imply that the derived probability of success will be an under-estimator for the actual recommendation success, which is confirmed in our experiments.

\section{Experiments}
\label{sec:experiments}
We composed and performed an experiment in the context of pandemic influenza, where we analyze the mitigation strategy to vaccinate a population when only a limited number of vaccine doses is available (details about the rationale behind this scenario in Section~\ref{subsec:bg:pandemicinfluenza}).
In our experiment, we accommodate a realistic setting to evaluate vaccine allocation, where we consider a large and realistic social network and a wide range of $R_0$ values.

We consider the scenario when a pandemic is emerging in a particular geographical region and vaccines becomes available, albeit in a limited number of doses. 
When the number of vaccine doses is limited, it is imperative to identify an optimal vaccine allocation strategy \cite{Medlock2009}.
In our experiment, we explore the allocation of vaccines over five different age groups, that can be easily approached by health policy officials: pre-school children, school-age children, young adults, older adults and the elderly, as proposed in \cite{chao2010}.

\subsection{Influenza model and configuration}
The epidemiological model used in the experiments is the FluTE stochastic individual-based model.
In our experiment we consider the population of Seattle (United States) that includes 560,000 individuals \cite{chao2010}.
This population is realistic both with respect to the number of individuals and its community structure, and provides an adequate setting for the validation of vaccine strategies \cite{Willem2014}.

At the first day of the simulated epidemic, 10 random individuals are seeded with an infection.
The epidemic is simulated for 180 days, during which time no more infections are seeded.
Thus, all new infections established during the run time of the simulation, result from the mixing between infectious and susceptible individuals.
We assume no pre-existing immunity towards the circulating virus variant.
We choose the number of vaccine doses to allocate to be approximately 4.5\% of the population size \cite{Medlock2009}.

We perform our experiment for a set of $R_0$ values within the range of $1.4$ to $2.4$, in steps of 0.2. This range is considered representative for the epidemic potential of influenza pandemics \cite{Basta2009,Medlock2009}. We refer to this set of $R_0$ values as $\mathcal{R}_0$.

Note that the setting described in this subsection, in conjunction with a particular $R_0$ value, corresponds to a model configuration (i.e., $c_0 \in \mathcal{C}$).

The computational complexity of FluTE simulations depends both on the size of the susceptible population and the proportion of the population that becomes infected.
For the population of Seattle, the simulation run time was up to $11\frac{1}{2}$ minutes (median of $10\frac{1}{2}$ minutes, standard deviation of 6 seconds), on state-of-the-art hardware (details in Supplementary Information, Section 7).

\subsection{Formulating vaccine allocation strategies}
We consider 5 age groups to which vaccine doses can be allocated: pre-school children (i.e., 0-4 years old), school-age children (i.e., 5-18 years old), young adults (i.e., 19-29 years old), older adults (i.e., 30-64 years old) and the elderly (i.e., $>$ 65 years old) \cite{chao2010}.
An allocation scheme can be encoded as a Boolean 5-tuple, where each position in the tuple corresponds to the respective age group.
The Boolean value at a particular position in the tuple denotes whether vaccines should be allocated to the respective age group.
When vaccines are to be allocated to a particular age group, this is done proportional to the size of the population that is part of this age group \cite{Medlock2009}.
To decide on the best vaccine allocation strategy, we enumerate all possible combinations of this tuple.

\subsection{Outcome distributions}
  To establish a proxy for the ground truth concerning the outcome distributions of the 32 considered preventive strategies, all strategies were evaluated $1000$ times, for each of the $R_0$ values in $\mathcal{R}_0$. We will use this ground truth as a reference to validate the correctness of the recommendations obtained throughout our experiments.

$\mathcal{R}_0$ presents us with an interesting evaluation problem. To demonstrate this, we visualize the outcome distribution for $R_0=1.4$ and for $R_0=2.4$ in Figure~\ref{fig:violin} (the outcome distributions for the other $R_0$ values are shown in Section 3 of the Supplementary Information).
Firstly, we observe that for different values of $R_0$, the distances between top arms' means differ.
Additionally, outcome distribution variances vary over the set of $R_0$ values in $\mathcal{R}_0$.
These differences produce distinct levels of evaluation hardness (see Section~\ref{subsec:methods:bestarmid}), and demonstrate the setting's usefulness as benchmark to evaluate preventive strategies.
While we discuss the hardness of the experimental settings under consideration, it is important to state that our best-arm identification framework requires no prior knowledge on the problem's hardness.
Secondly, we expect the outcome distribution to be bimodal. However, the probability to sample from the mode of the outcome distribution that represents the non-established epidemic decreases as $R_0$ increases \cite{Lloyd2005}.
This expectation is confirmed when we inspect Figure~\ref{fig:violin}, the left panel shows a bimodal distribution for $R_0=1.4$, while the right panel shows a unimodal outcome distribution for $R_0=2.4$, as only samples from the established epidemic were obtained.

\begin{figure}
  \centering
  \begin{minipage}[t]{0.45\textwidth}
    \includegraphics[width=\linewidth]{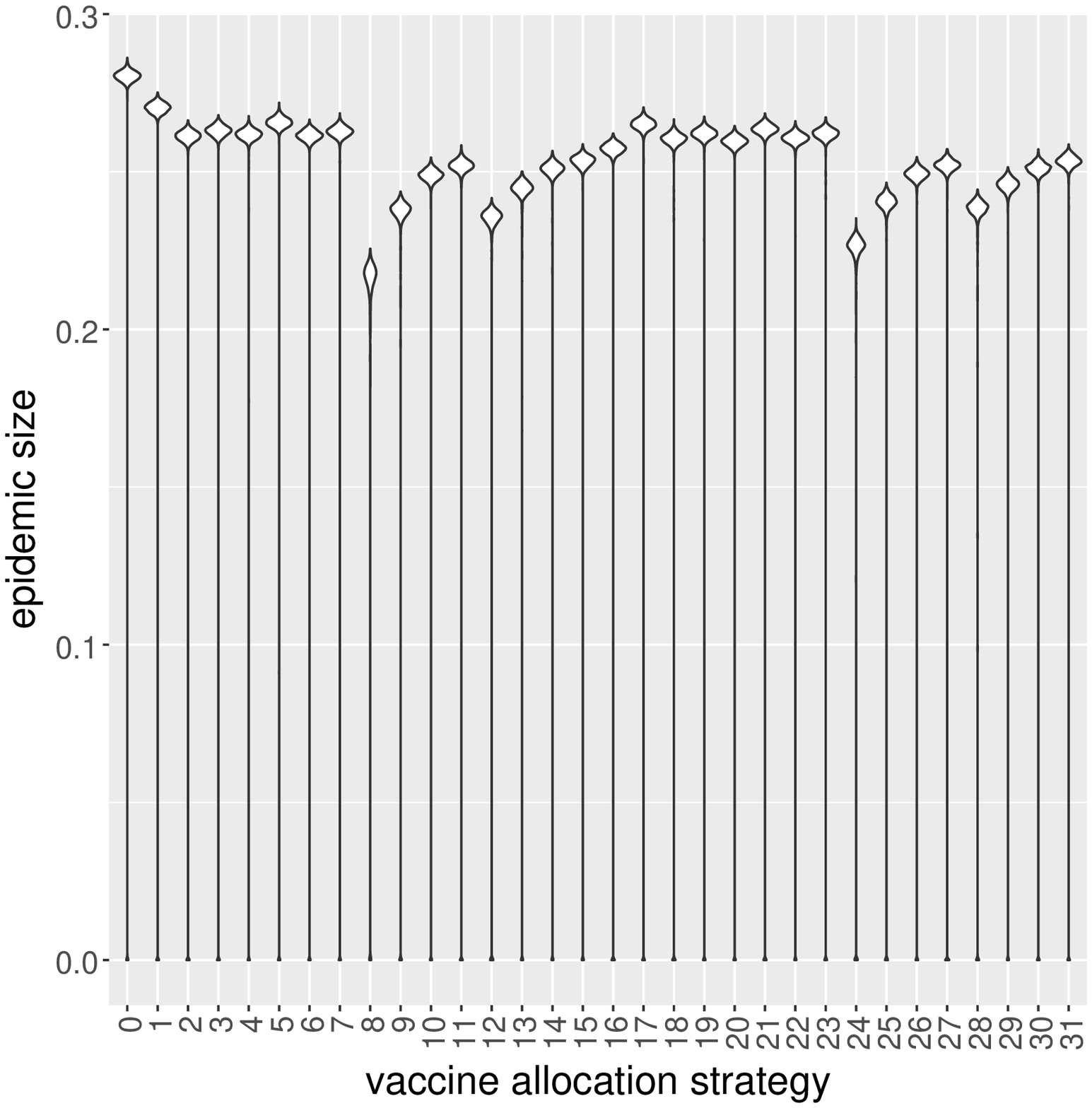}
  \end{minipage}
  \begin{minipage}[t]{0.45\textwidth}
    \includegraphics[width=\linewidth]{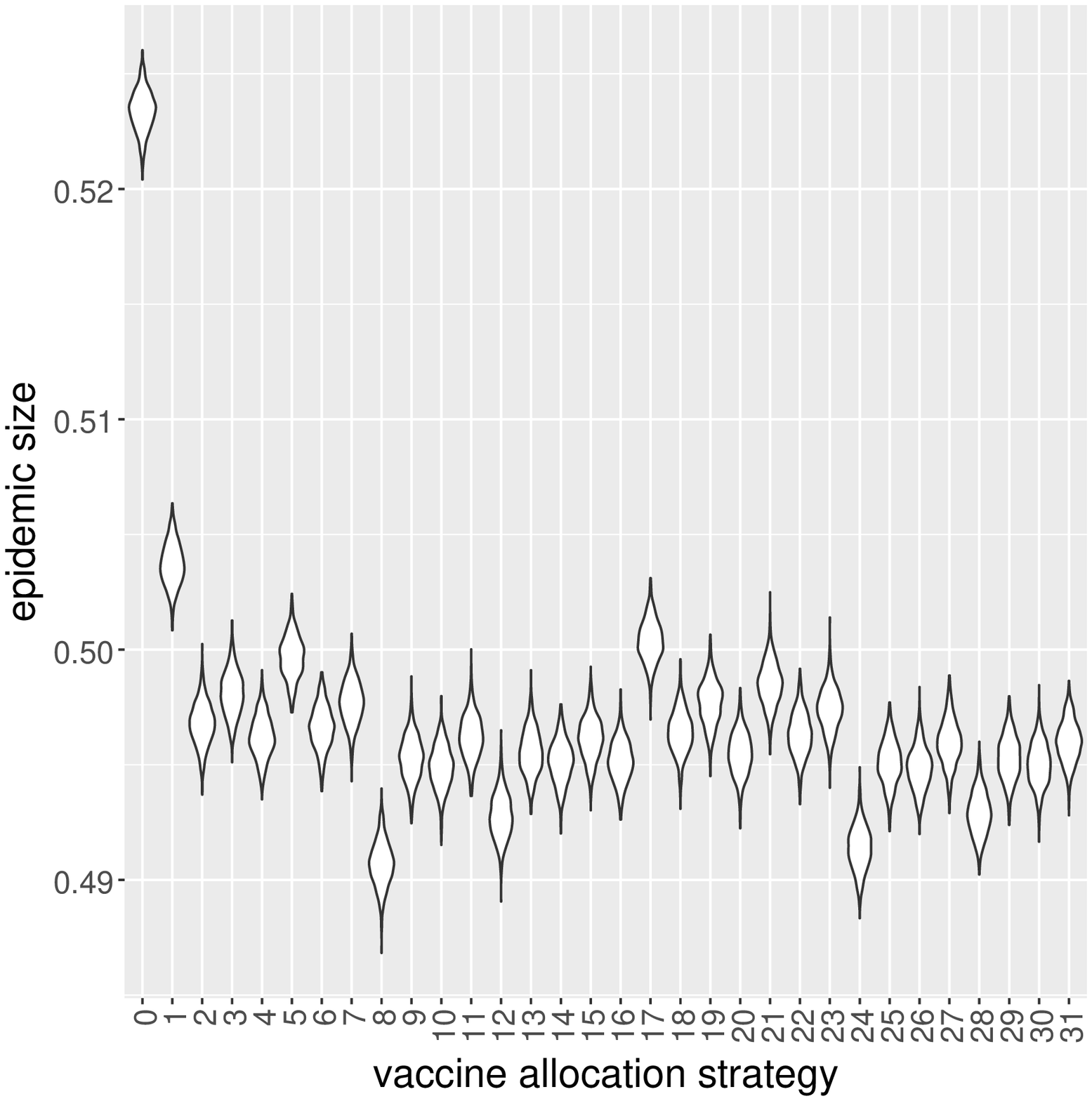}
  \end{minipage}
\caption{Violin plot that depicts the density of the outcome distribution (i.e., epidemic size) for 32 vaccine allocation strategies (left panel $R_o=1.4$, right panel $R_o=2.4$).
}
\label{fig:violin}
\end{figure}

Our analysis identified that the best vaccine allocation strategy was $\langle0,1,0,0,0\rangle$ (i.e., allocate vaccine to school children, strategy 8) for all $R_0$ values in $\mathcal{R}_0$.

\subsection{Best-arm identification experiment}
To assess the performance of the different best-arm identification algorithms (i.e., Successive Rejects, BayesGap and Top-two Thompson sampling) we run each algorithm for all budgets in the range of 32 to 500.
This evaluation is performed on the influenza bandit game that we defined earlier.
For each budget, we run the algorithms 100 times, and report the recommendation success rate.
In the previous section, the optimal vaccine allocation strategy was identified to be $\langle0,1,0,0,0\rangle$ (i.e., vaccine allocation strategy 8) for all $R_0$ in $\mathcal{R}_0$.
We thus consider a recommendation to be correct when it equals this vaccine allocation strategy.

We evaluate the algorithm's performance with respect to each other and with respect to uniform sampling, the current state-of-the art to evaluate preventive strategies.
The uniform sampling method pulls arm $A_u$ for each step $t$ of the given budget $T$, where $A_u$'s index $u$ is sampled from the uniform distribution $\mathcal{U}(1,K)$.
To consider different levels of hardness, we perform this analysis for each $R_0$ value in $\mathcal{R}_0$.

For the Bayesian best-arm identification algorithms, the prior specifications are detailed in Section~\ref{subsec:methods:bestarmid}.
BayesGap requires an upper and lower bound that is defined in terms of the used posteriors.
In our experiments, we use upper bound $U_k(t)$ and lower bound $L_k(t)$ that were established in Section~\ref{subsec:methods:bestarmid}.
Top-two Thompson sampling requires a parameter that modulates the amount of exploration $\omega$. As it is important for best-arm identification algorithms to differentiate between the top two arms, we choose $\omega=0.5$, such that, in the limit, Top-two Thompson sampling will explore the top two arms uniformly.

We censor the reward distribution based on the threshold $T_0$ we defined in Section~\ref{subsec:methods:threshold}.
This threshold depends on basic reproductive number $R_0$ and dispersion parameter $\gamma$.
$R_0$ is defined explicitly for each of our experimental settings. For the dispersion parameter we choose $\gamma=0.5$, which is a conservative choice according to the literature \cite{dorigatti2012,fraser2011}.
We define the probability cutoff $\ell=10^{-10}$.

Figure~\ref{fig:bandit_run} shows recommendation success rate for each of the best-arm identification algorithms for $R_0=1.4$ (left panel) and $R_0=2.4$ (right panel).
The results for the other $R_0$ values are visualized in Section 4 of the Supplementary Information.
The results for different values of $R_0$ clearly indicate that our selection of best-arm identification algorithms significantly outperforms the uniform sampling method.
Overall, the uniform sampling method requires more than double the amount of evaluations to achieve a similar recommendation performance. For the harder problems (e.g., setting with $R_0=2.4$), recommendation uncertainty remains considerable even after consuming 3 times the budget required by Top-two Thompson sampling.

All best-arm identification algorithms require an initialization phase in order to output a well-defined recommendation.
Successive Rejects needs to pull each arm at least once, while Top-two Thompson sampling and BayesGap need to pull each arm respectively 2 and 3 times (details in Section~\ref{subsec:methods:bestarmid}).
For this reason, these algorithms' performance can only be evaluated after this initialization phase.
BayesGap's performance is on par with Successive Rejects, except for the hardest setting we studied (i.e., $R_0=2.4$).
In comparison, Top-two Thompson sampling consistently outperforms Successive Rejects 30 pulls after the initialization phase.

Top-two Thompson sampling needs to initialize each arm's posterior with 2 pulls, i.e., double the amount of uniform sampling and Successive Rejects. However, our experiments clearly show that none of the other algorithms reach any acceptable recommendation rate using less than 64 pulls.

\begin{figure}
  \centering
  \begin{minipage}[t]{0.45\textwidth}
    \includegraphics[width=\linewidth]{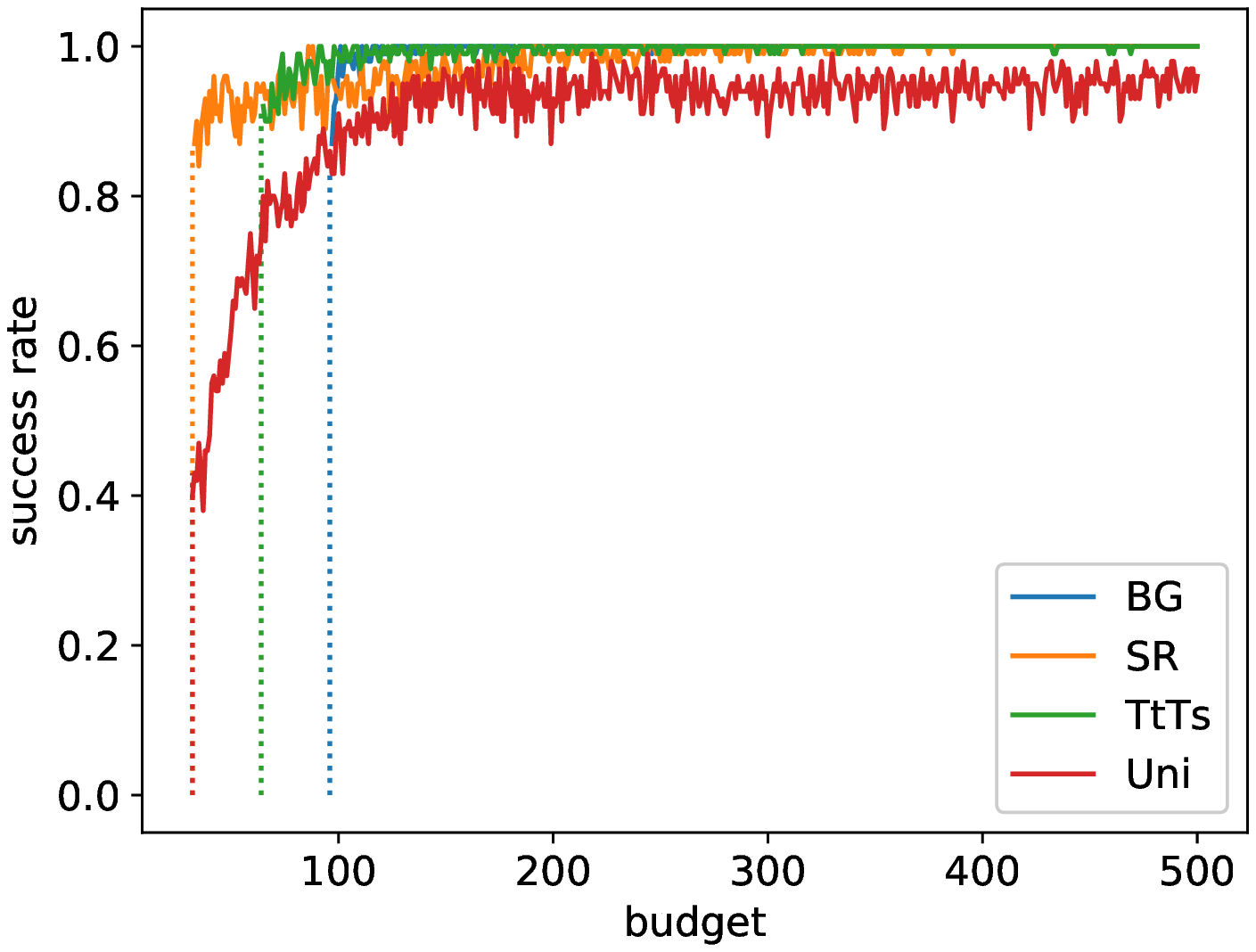}
  \end{minipage}
  \begin{minipage}[t]{0.45\textwidth}
    \includegraphics[width=\linewidth]{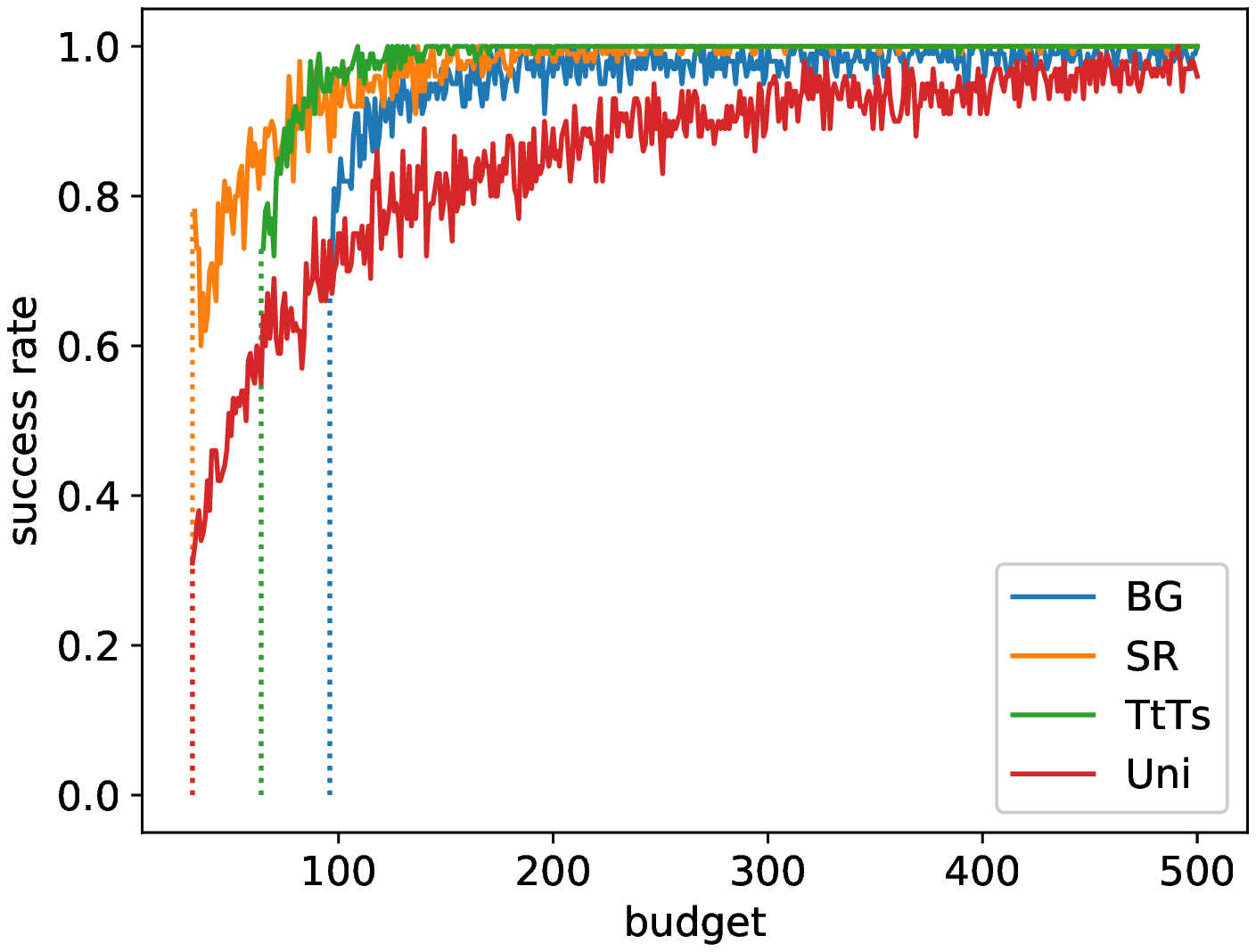}
  \end{minipage}
\caption{In this figure, we present the results for the experiment with $R_0=1.4$ (left panel) and $R_0=2.4$ (right panel). Each curve represents the rate of successful arm recommendations (y-axis) for a range of budgets (x-axis). A curve is shown for each of the considered algorithms: BayesGap (legend: BG), Successive Rejects (legend: SR), Top-two Thompson sampling (legend: TtTs) and Uniform sampling (legend: Uni).
}
\label{fig:bandit_run}
\end{figure}

In Section~\ref{sec:methods} we derived a statistic to express the probability of success ($P_s$) concerning a recommendation made by Top-two Thompson sampling. We analyzed this probability for all the Top-two Thompson sampling recommendations that were obtained in the experiment described above. To provide some insights on how this statistic can be used to support policy makers, we show the $P_s$ values of all Top-two Thompson sampling recommendations for $R_0=2.4$ in the left panel of Figure~\ref{fig:prob_of_success_24} (Figures for the other $R_0$ values in Section 5 of the Supplementary Information). This Figure indicates that $P_s$ closely follows recommendation correctness and that the uncertainty of $P_s$ is inversely proportional to the size of the available budget.
Additionally, in the right panel of Figure~\ref{fig:prob_of_success_24} (Figures for the other $R_0$ values in Section 6 of the Supplementary Information) we confirm that $P_s$ underestimates recommendation correctness.
These observations show that $P_s$ has the potential to serve as a conservative statistic to inform policy makers about the confidence of a particular recommendation, and thus can be used to define meaningful cutoffs to guide policy makers in their interpretation of the recommendation of preventive strategies. 

\begin{figure}
  \centering
  \begin{minipage}[t]{0.45\textwidth}
    \includegraphics[width=\linewidth]{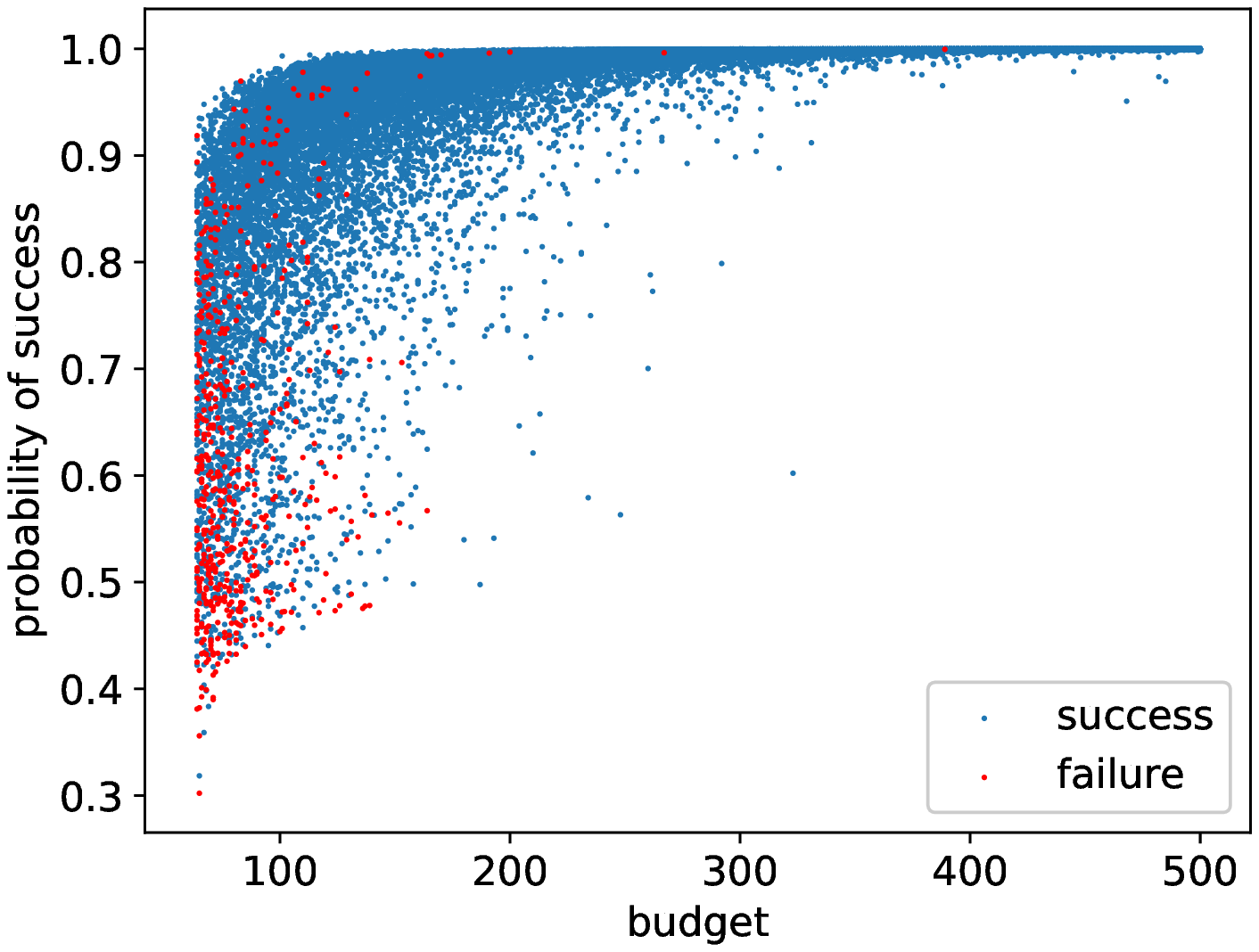}
  \end{minipage}
  \begin{minipage}[t]{0.45\textwidth}
    \includegraphics[width=\linewidth]{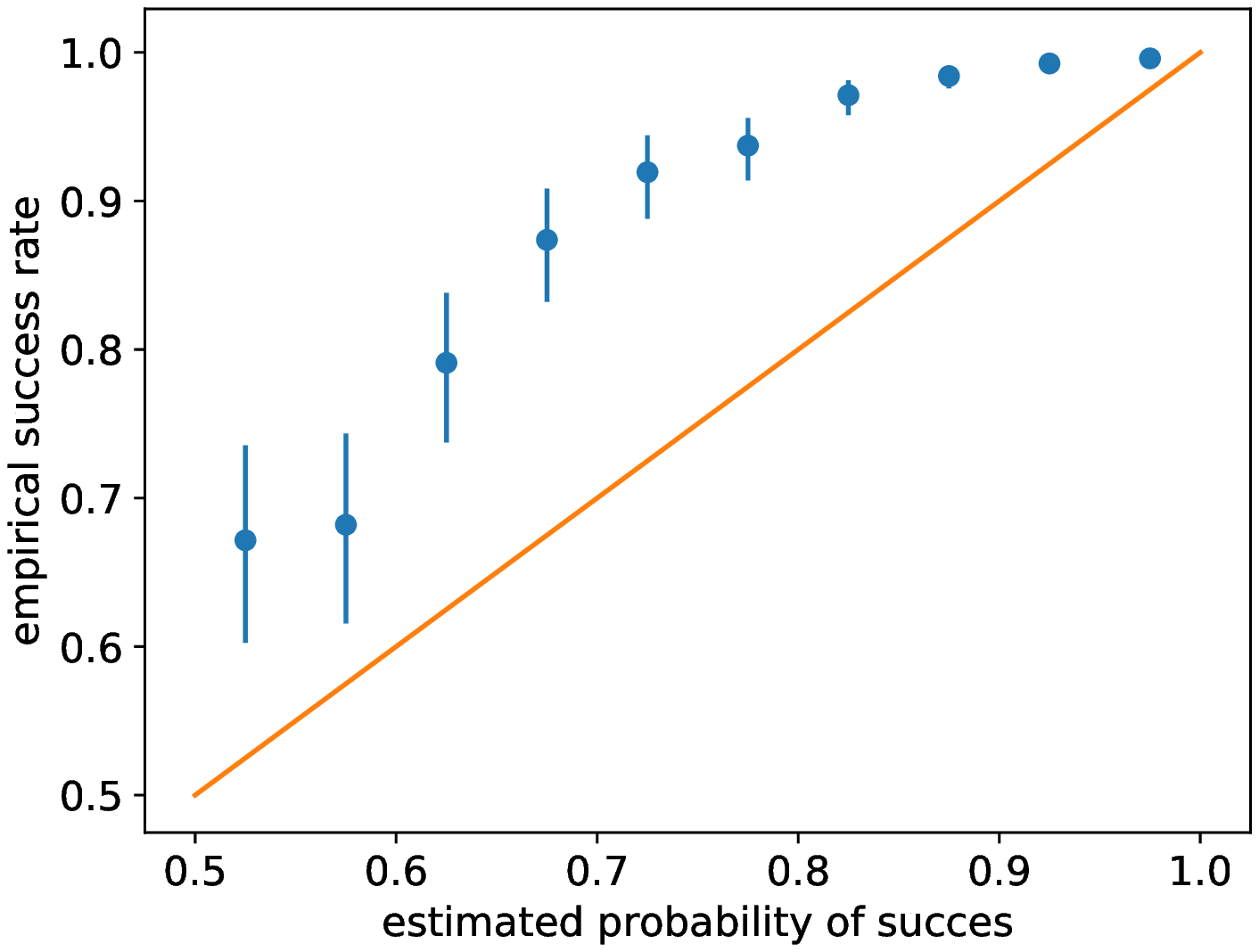}
  \end{minipage}
\caption{Top-two Thompson sampling was run 100 times for each budget for the experiment with $R_0=2.4$. For each of the recommendations, $P_s$ was computed. In the left panel, these $P_s$ values are shown as a scatter plot, where each point's color reflects the correctness of the recommendation (see legend). In the right panel, the $P_s$ values were binned (i.e., 0.5 to 1 in steps of 0.05). Per bin, we thus have a set of Bernoulli trials, for which we show the empirical success rate (blue scatter) and the Clopper-Pearson confidence interval (blue confidence bounds). The orange reference line denotes perfect correlation between the empirical success rate and the estimated probability of success.
}
\label{fig:prob_of_success_24}
\end{figure}

\section{Conclusion}
\label{sec:conclusion}
We formulate the objective to select the best preventive strategy in an individual-based model as a fixed budget best-arm identification problem. We set up an experiment to evaluate this setting in the context of a realistic pandemic influenza. To assess the best arm recommendation performance of the preventive bandit, we report a success rate over 100 independent bandit runs.

We demonstrate that it is possible to efficiently identify the optimal preventive strategy using only a limited number of model evaluations, even if there is a large number of preventive strategies to consider.
Compared to uniform sampling, our technique is able to recommend the best preventive strategy reducing the number of required model evaluations 2-to-3 times, when using Top-two Thompson sampling.
Additionally, we defined a statistic to support policy makers with their decisions, based on the posterior information obtained during Top-two Thompson sampling. As such, we present a decision support tool to assist policy makers to mitigate epidemics.
Our framework will enable the use of individual-based models in studies where it would otherwise be computationally too prohibitive, and allow researchers to explore a wider variety of model scenarios.

In this paper, we learn with respect to a single model outcome (i.e., single objective). However, for many pathogens it can be interesting to incorporate multiple objectives (e.g., morbidity, mortality, financial cost). Therefore, in future work, we aim to use \emph{multi-objective multi-armed bandits}.

\section*{Acknowledgments}
Pieter Libin and Timothy Verstraeten were supported by a PhD grant of the FWO (Fonds Wetenschappelijk Onderzoek - Vlaanderen).
Kristof Theys, Jelena Grujic and Diederik Roijers were supported by a postdoctoral grant of the FWO.
The computational resources were provided by an EWI-FWO grant (Theys, KAN2012 1.5.249.12.).

\pagebreak
\begin{center}
\textbf{\large Supplementary Information}
\end{center}
\setcounter{equation}{0}
\setcounter{figure}{0}
\setcounter{table}{0}
\setcounter{page}{1}
\setcounter{section}{0}
\makeatletter
\renewcommand{\theequation}{S\arabic{equation}}
\renewcommand{\thesection}{S\arabic{section}}
\renewcommand{\thefigure}{S\arabic{figure}}

\section{Introduction}
In this Supplementary Information we provide a proof for BayesGap's simple regret bound (Section 2). Furthermore, we provide additional figures that were omitted from the main manuscript: figures for the outcome (i.e., epidemic size) distributions (Section 3), figures for the experimental success rates (Section 4), figures for the probabilities of success (i.e., $P_s$ values) per budget  (Section 5) and figures for the binned distribution over $P_s$ values (Section 6). Finally, in Section 7, we describe the computational resources that were used to execute the simulations. 

\section{BayesGap simple regret bound for T-distributed posteriors}
\begin{lemma}
Consider a Jeffrey's prior $(\mu_k, \sigma^2_k) \sim \sigma_k^{-3}$over the parameters of the Gaussian reward distributions. Then the posterior mean of arm $k$ has the following nonstandardized t-distribution at pull $n_k$:
\begin{equation*}
\mu_k\ |\ \bar{x}_{k,n_k}, S_{k,n_k} \sim \mathcal{T}_{n_k}(\bar{x}_{k,n_k}, n_k^{-1}\sqrt{S_{k,n_k}})
\end{equation*}
where $n_k$ is the number of pulls for arm $k$, $\bar{x}_{k, n_k}$ is the sample mean and $S_{k, n_k}$ is the sum of squares.
\label{lemma:post_mean}
\end{lemma}
\begin{proof}
This lemma was presented and proved by Honda et al. \cite{Honda2014}.
\end{proof}

\begin{lemma}
Consider a random variable $X \sim \mathcal{T}_\nu(\mu, \lambda)$ with variance $\sigma^2 = \frac{\nu}{\nu - 2}\lambda^2$, $\nu > 2$ and $\beta > 0$. The probability that $X$ is within a radius $\beta \sigma$ from its mean can then be written as:\\
\begin{equation*}
\begin{split}
P(|X - \mu| < \beta\sigma) \ge 1 - 2\frac{\nu}{\nu-1} \frac{C(\nu)}{\beta} \left(1 + \frac{\beta^2}{\nu}\right)^{-0.5(\nu-1)}
\end{split}
\end{equation*}
where\\
\begin{equation*}
\begin{split}
C(\nu) = \frac{\Gamma(0.5\nu+0.5)}{\Gamma(0.5\nu) \sqrt{\pi \nu}}
\end{split}
\end{equation*}
is the normalizing constant of a standard t-distribution.
\label{lemma:upper_bound}
\end{lemma}
\begin{proof}
Consider a random variable $Z \sim \mathcal{T}_\nu(0, 1)$, $\nu > 2$ and $\beta > 0$. Then the probability of $Z$ being greater than $\beta \sqrt{\frac{\nu}{\nu - 2}}$ is:
\begin{equation*}
\begin{split}
P(Z > \beta\sqrt{\frac{\nu}{\nu - 2}}) &\stackrel{(1)}{=} \int^{+\infty}_{\beta\sqrt{\frac{\nu}{\nu - 2}}} \mathcal{T}_\nu(z\ |\ 0, 1) dz\\
	&= C(\nu) \int^{+\infty}_{\beta\sqrt{\frac{\nu}{\nu - 2}}} \left(1 + \frac{z^2}{\nu}\right)^{-0.5(\nu+1)} dz\\
	&\stackrel{(2)}{\le} C(\nu) \int^{+\infty}_{\beta\sqrt{\frac{\nu}{\nu-2}}} \frac{z}{\beta\sqrt{\frac{\nu}{\nu - 2}}}\left(1 + \frac{z^2}{\nu}\right)^{-0.5(\nu+1)} dz\\
	&= -\frac{\nu}{\nu-1} \frac{\sqrt{\nu - 2}}{\beta\sqrt{\nu}} C(\nu) \int^{+\infty}_{\beta\sqrt{\frac{\nu}{\nu-2}}} -\frac{\nu-1}{\nu} z\left(1 + \frac{z^2}{\nu}\right)^{-0.5(\nu+1)} dz\\
	&\stackrel{(3)}{=} \left. -\frac{\sqrt{\nu(\nu - 2)}}{\nu-1} \frac{C(\nu)}{\beta} \left(1 + \frac{z^2}{\nu}\right)^{-0.5(\nu-1)} \right|^{+\infty}_{\beta\sqrt{\frac{\nu}{\nu-2}}}\\
	&\stackrel{(4)}{=} \frac{\sqrt{\nu(\nu - 2)}}{\nu-1} \frac{C(\nu)}{\beta} \left(1 + \frac{\beta^2}{\nu - 2}\right)^{-0.5(\nu-1)}
\end{split}
\end{equation*}
The probability of $Z$ being greater than the lower bound $\beta \sqrt{\frac{\nu}{\nu-2}}$ is the integral over its probability density function, starting from that lower bound (1). In the integral, we introduce a factor $\frac{z}{\beta \sqrt{\frac{\nu}{\nu-2}}}$, which is greater than $1$ for the considered values of $z$ (2). We then take note of the following derivative, and use this result to analytically solve the integral (3):
\begin{equation*}
\frac{d}{dx}\left(1 + \frac{x^2}{\nu}\right)^{-0.5(\nu-1)} = - \frac{\nu-1}{\nu} x \left(1 + \frac{x^2}{\nu}\right)^{-0.5(\nu+1)}
\end{equation*}
Finally, we solve the primitive from $\frac{z}{\beta \sqrt{\frac{\nu}{\nu-2}}}$ to infinity (4).

Next, we apply a union bound to obtain a lower bound on the probability that the magnitude of $Z$ is smaller than $\beta\sqrt{\frac{\nu}{\nu-2}}$:
\begin{equation*}
\begin{split}
P(|Z| < \beta\sqrt{\frac{\nu}{\nu-2}}) \ge 1 - 2\frac{\sqrt{\nu(\nu - 2)}}{\nu-1} \frac{C(\nu)}{\beta} \left(1 + \frac{\beta^2}{\nu - 2}\right)^{-0.5(\nu-1)}
\end{split}
\end{equation*}
Finally, consider $Z = \frac{(X - \mu)}{\lambda}$:
\begin{equation*}
\begin{split}
P(|X - \mu| < \beta\sqrt{\frac{\nu}{\nu-2}}\lambda) \ge 1 - 2\frac{\sqrt{\nu(\nu - 2)}}{\nu-1} \frac{C(\nu)}{\beta} \left(1 + \frac{\beta^2}{\nu - 2}\right)^{-0.5(\nu-1)}
\end{split}
\end{equation*}
\end{proof}

\begin{lemma}
Consider a $K$-armed bandit problem with budget $T$ and $K$ arms. Let $U_k(t)$ and $L_k(t)$ be upper and lower bounds that hold for all times $t \le T$ and all arms $k \le K$ with probability $1 - \delta_k(t)$. Finally, let $g_k$ be a monotonically decreasing function such that $U_k(t) - L_k(t) \le g_k(n_k(t-1))$ and $\sum\limits^K_{k=1} g_k^{-1}(H_{k,\epsilon}) \le T - K$. We can then bound the simple regret $R_T$ as:
\begin{equation*}
\begin{split}
P(R_{T} < \epsilon) \ge 1 - \sum\limits^K_{k=1} \sum\limits^T_{t=1} \delta_{k}(t)
\end{split}
\end{equation*}
\label{lemma:regret}
\end{lemma}
\begin{proof}
First, we define $\mathcal{E}$ as the event in which every mean $\mu_k$ is bounded by its associated bounds (i.e., $U_k(t)$ and $L_k(t)$) for each time step \cite{hoffman2014correlation}.
\begin{equation*}
\mathcal{E} := \forall k \le K, \forall t \le T: L_k(t) \le \mu_k \le U_k(t)
\end{equation*}
The probability of $\mu_k$ deviating from a single bound at time $t$ is by definition $\delta_{k}(t)$. When applying the union bound, we obtain $P(\mathcal{E})\ge 1 - \sum\limits^K_{k=1} \sum\limits^T_{t=1} \delta_{k}(t)$.
The probability of regret is equal to the probability of the event $\mathcal{E}$ occuring, as proven in \cite{hoffman2014correlation}.
\end{proof}

\begin{theorem}
Consider a $K$-armed Gaussian bandit problem with budget $T$ and unknown variance. Let $\sigma_G^2$ be a generalization of that variance over all arms, and $U_k(t)$ and $L_k(t)$ respectively be the upper and lower bounds for each arm $k$ at time $t$, where $U_k(t) = \hat{\mu}_k(t) + \beta \hat{\sigma}_k(t)$ and $L_k(t) = \hat{\mu}_k(t) - \beta \hat{\sigma}_k(t)$. The simple regret is then bounded as:
\begin{equation*}
\begin{split}
P(R_{T} \le \epsilon) &\ge 1 - 2\sum\limits^K_{k=1} \sum\limits^T_{t=1} \frac{\sqrt{n_{k}(t)(n_{k}(t) - 2)}}{n_{k}(t)-1} \frac{C(n_{k}(t))}{\beta}\left(1 + \frac{\beta^2}{n_{k}(t) - 2}\right)^{-0.5(n_{k}(t)-1)}\\
&\ge 1 - O\left(K T \left(1 + \frac{\beta^2}{\min\limits_{k,t} n_{k}(t)}\right)^{-0.5\min\limits_{k,t}n_{k}(t)}\right)
\end{split}
\end{equation*}
where:
\begin{equation*}
\beta = \sqrt{\frac{T - 3K}{4 H_{\epsilon} \sigma_G^2}}
\end{equation*}
Note that when $\min\limits_{k,t} n_{k}(t) \rightarrow +\infty$, the bound decreases exponentially in $\beta$, similar to the problem setting presented in \cite{hoffman2014correlation}. Intuitively, this result makes sense, as for known variances, a Gaussian can be used to describe the posterior means, and indeed, as the number of pulls approaches infinity, our t-distributions converge to Gaussians.
\end{theorem}
\begin{proof}
According to Lemma~\ref{lemma:post_mean}, the posterior over the average reward is a t-distribution with scaling factor $\lambda_k(t) = n_k(t)^{-1}\sqrt{S_{k,n_k(t)}}$.
Therefore,
\begin{equation*}
\begin{split}
U_k(t+1) - L_k(t+1) &= 2\beta \hat{\sigma}_k(t)\\
&\stackrel{(1)}{=} 2\beta \sqrt{n_{k}(t)(n_{k}(t)-2)^{-1}\lambda_k(t)^2}\\
&\stackrel{(2)}{=} \sqrt{n_{k}(t)(n_{k}(t)-2)^{-1}n_{k}(t)^{-2}S_{k, n_k(t)}}\\
&= \sqrt{(n_{k}(t)-2)^{-1}\frac{S_{k, n_k(t)}}{n_{k}(t)}}\\
&\stackrel{(3)}{=} \sqrt{(n_k(t) - 2)^{-1} s^2_{k}(t)}\\
&\stackrel{(4)}{=} g_k(n_k(t))\\
\end{split}
\end{equation*}
The variance of a t-distribution equals $\frac{n_{k}(t)}{n_{k}(t)-2}\lambda_k(t)^2$ for arm $k$ at time $t$, with scaling factor $\lambda_k(t)$ as described in Lemma~\ref{lemma:post_mean} (1 + 2). We denote the variance over rewards per arm as $s^2_k(t)$ (3) and define $g_k(n_k(t))$ to be the upper bound expression as specified in Lemma~\ref{lemma:regret} (4).

Next, we compute the inverse of $g_k(n)$:
\begin{equation*}
\begin{split}
g_k^{-1}(m) &= \frac{4\beta^2s^2_k(t)}{m^2} + 2\\
\end{split}
\end{equation*}
We generalize $s^2_k(t)$ to a variance $\sigma_G^2$ representative for all arms.\footnote{In the main paper, we choose $\sigma_G^2 = \bar{s}^2_G$ to be the mean over all arm-specific variances obtained after the initialization phase.} 
Approximating the hardness of the problem as $H_\epsilon = \sum_k H_{k,\epsilon}^{-2}$, where $H_{k,\epsilon}$ is the arm-dependent hardness defined in \cite{hoffman2014correlation}, we obtain $\beta$ as follows:
\begin{equation*}
\begin{split}
&\sum^K_{k=1} g_k^{-1}(H_{k\epsilon}) \approx 4\beta^2\sigma_G^2 H_{\epsilon} + 2K = T - K\\
&\Leftrightarrow \beta = \sqrt{\frac{T - 3K}{4 H_{\epsilon} \sigma_G^2}}\\
\end{split}
\end{equation*}
Finally, as the conditions in Lemma~\ref{lemma:regret} on the function $g_k$ are now satisfied, the simple regret bound can be obtained using Lemma~\ref{lemma:regret} and the probability that the true mean is out of the arm-specific bounds $U_k(t)$ and $L_k(t)$, given in Lemma~\ref{lemma:upper_bound}.
\end{proof}

\newpage
\section{Outcome (i.e., epidemic size) distributions}
\begin{figure}[!h] 
  \begin{subfigure}[b]{0.5\linewidth}
    \centering
    \includegraphics[width=0.75\linewidth]{violin_14} 
    \caption{Outcome distributions for $R_{0}=1.4$.} 
    \vspace{4ex}
  \end{subfigure}
  \begin{subfigure}[b]{0.5\linewidth}
    \centering
    \includegraphics[width=0.75\linewidth]{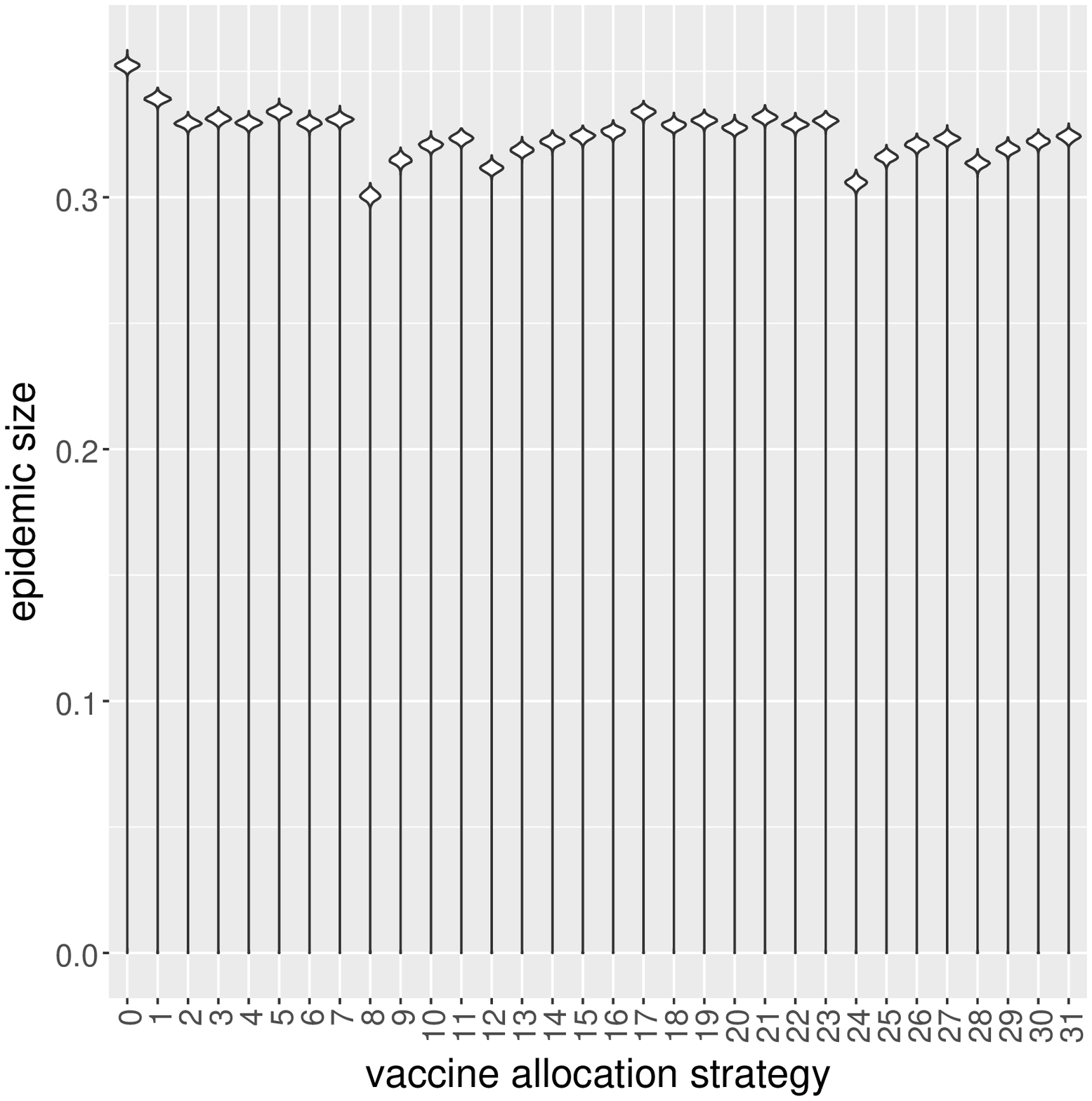} 
   \caption{Outcome distributions for $R_{0}=1.6$.}
    \vspace{4ex}
  \end{subfigure} 
  \begin{subfigure}[b]{0.5\linewidth}
    \centering
    \includegraphics[width=0.75\linewidth]{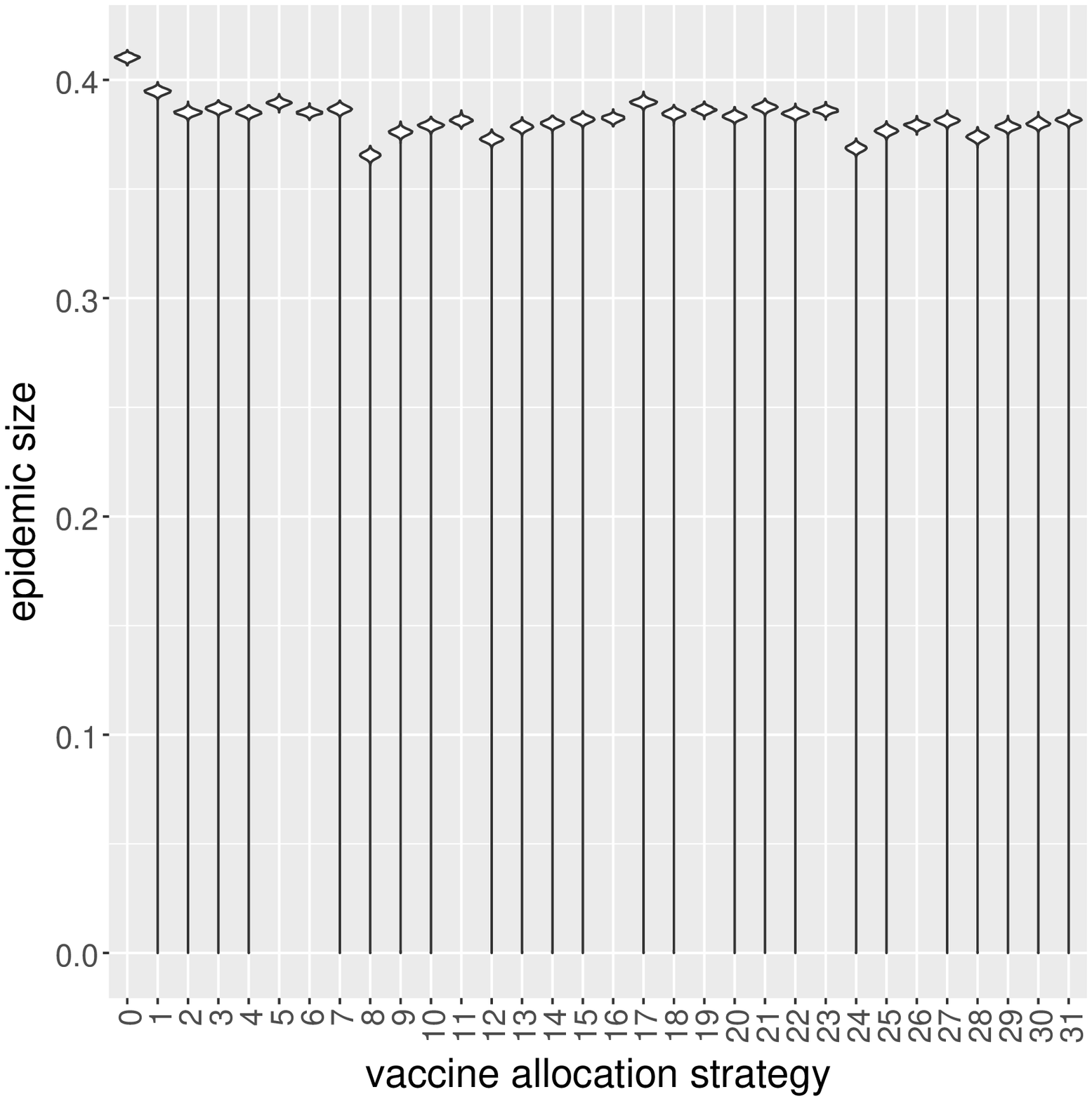} 
    \caption{Outcome distributions for $R_{0}=1.8$.}
  \end{subfigure}
  \begin{subfigure}[b]{0.5\linewidth}
    \centering
    \includegraphics[width=0.75\linewidth]{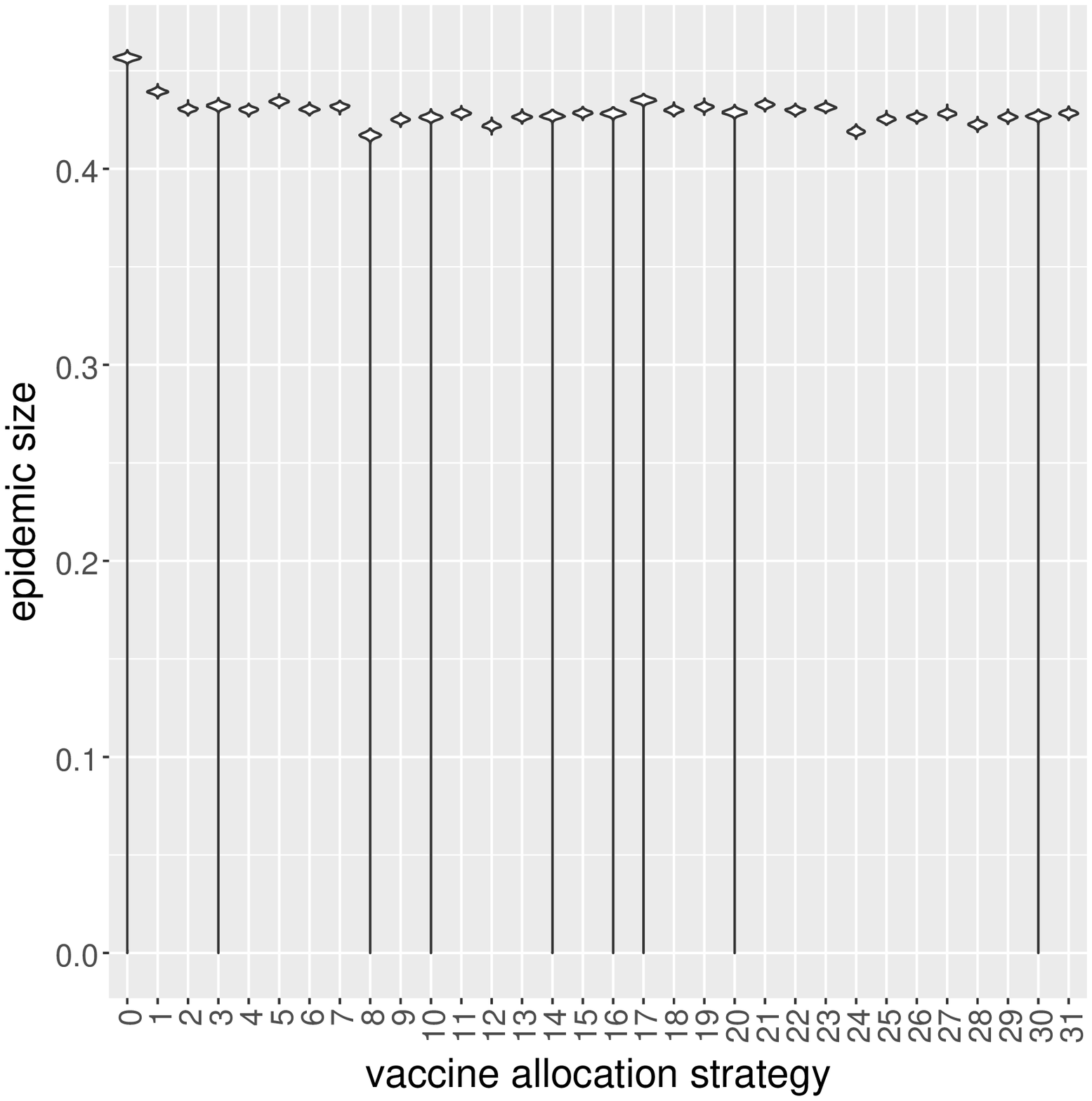}
    \caption{Outcome distributions for $R_{0}=2.0$.}
  \end{subfigure}
   \begin{subfigure}[b]{0.5\linewidth}
    \centering
    \includegraphics[width=0.75\linewidth]{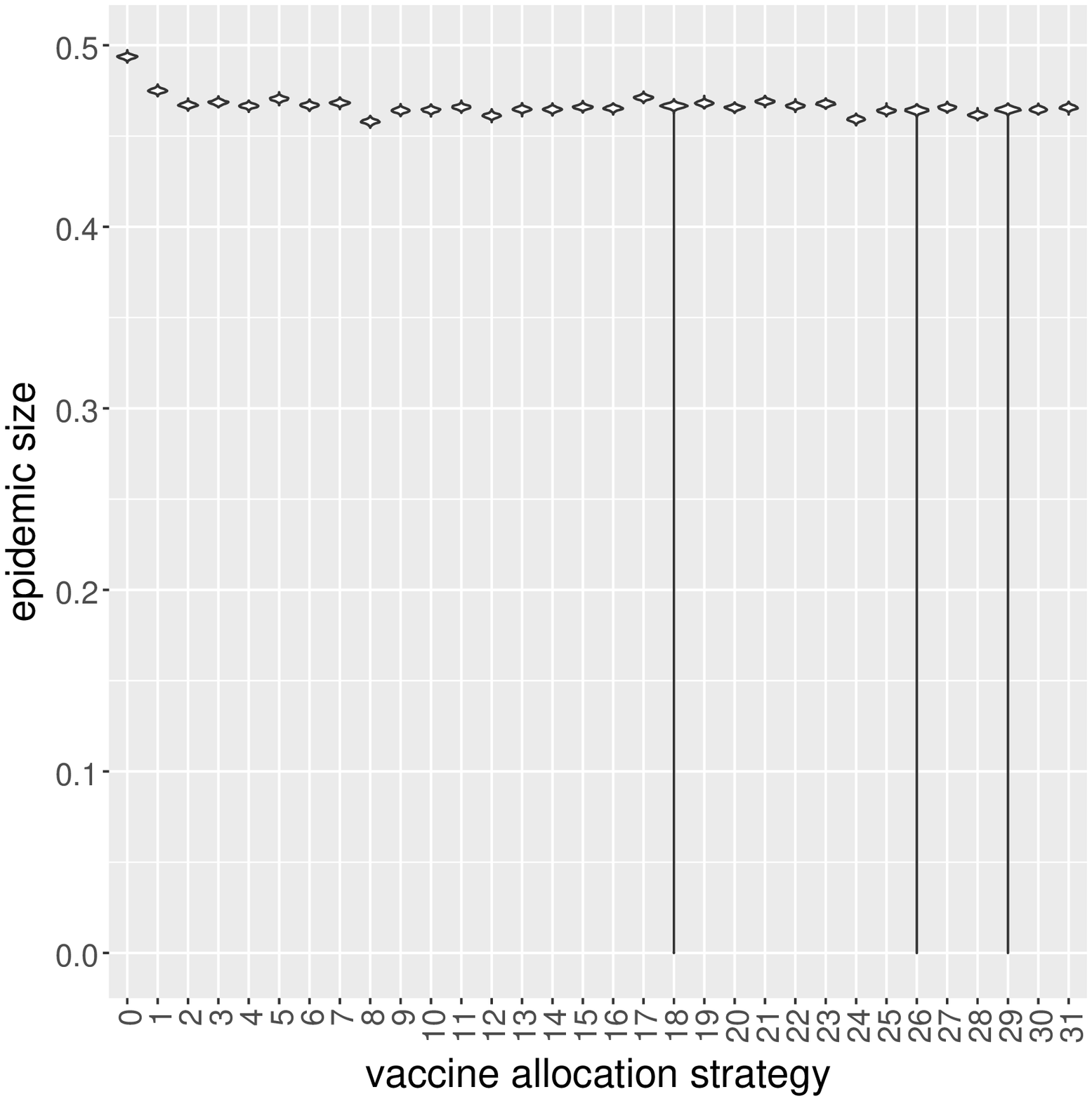}
    \caption{Outcome distributions for $R_{0}=2.2$.}
   \end{subfigure}
    \begin{subfigure}[b]{0.5\linewidth}
    \centering
    \includegraphics[width=0.75\linewidth]{violin_24}
    \caption{Outcome distributions for $R_{0}=2.4$.}
  \end{subfigure} 
\end{figure}

\newpage
\section{Bandit run success rates}
\begin{figure}[!h] 
  \begin{subfigure}[b]{0.5\linewidth}
    \centering
    \includegraphics[width=0.75\linewidth]{bandit_run_14} 
    \caption{Bandit run results for $R_{0}=1.4$.} 
    \vspace{4ex}
  \end{subfigure}
  \begin{subfigure}[b]{0.5\linewidth}
    \centering
    \includegraphics[width=0.75\linewidth]{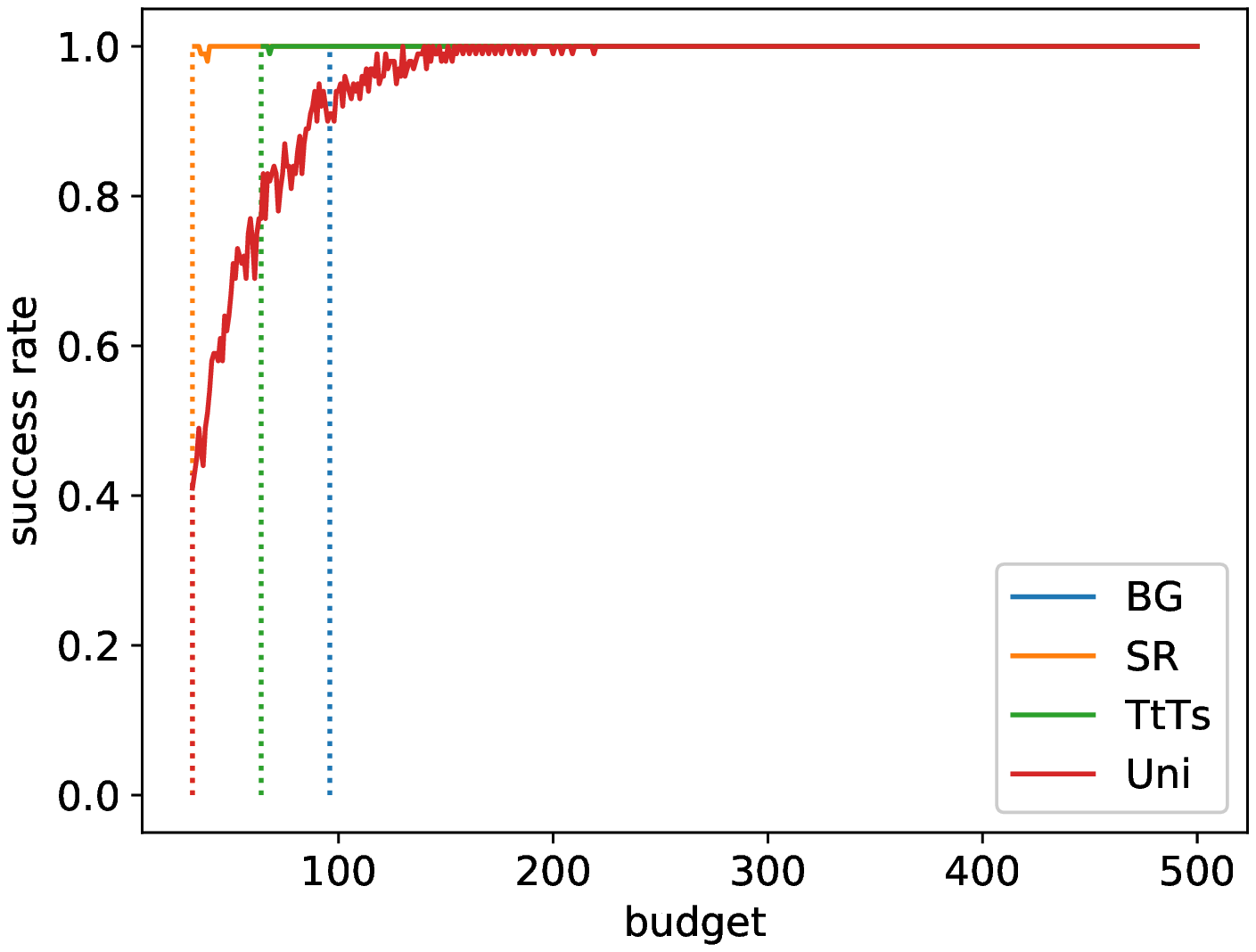} 
   \caption{Bandit run results for $R_{0}=1.6$.}
    \vspace{4ex}
  \end{subfigure} 
  \begin{subfigure}[b]{0.5\linewidth}
    \centering
    \includegraphics[width=0.75\linewidth]{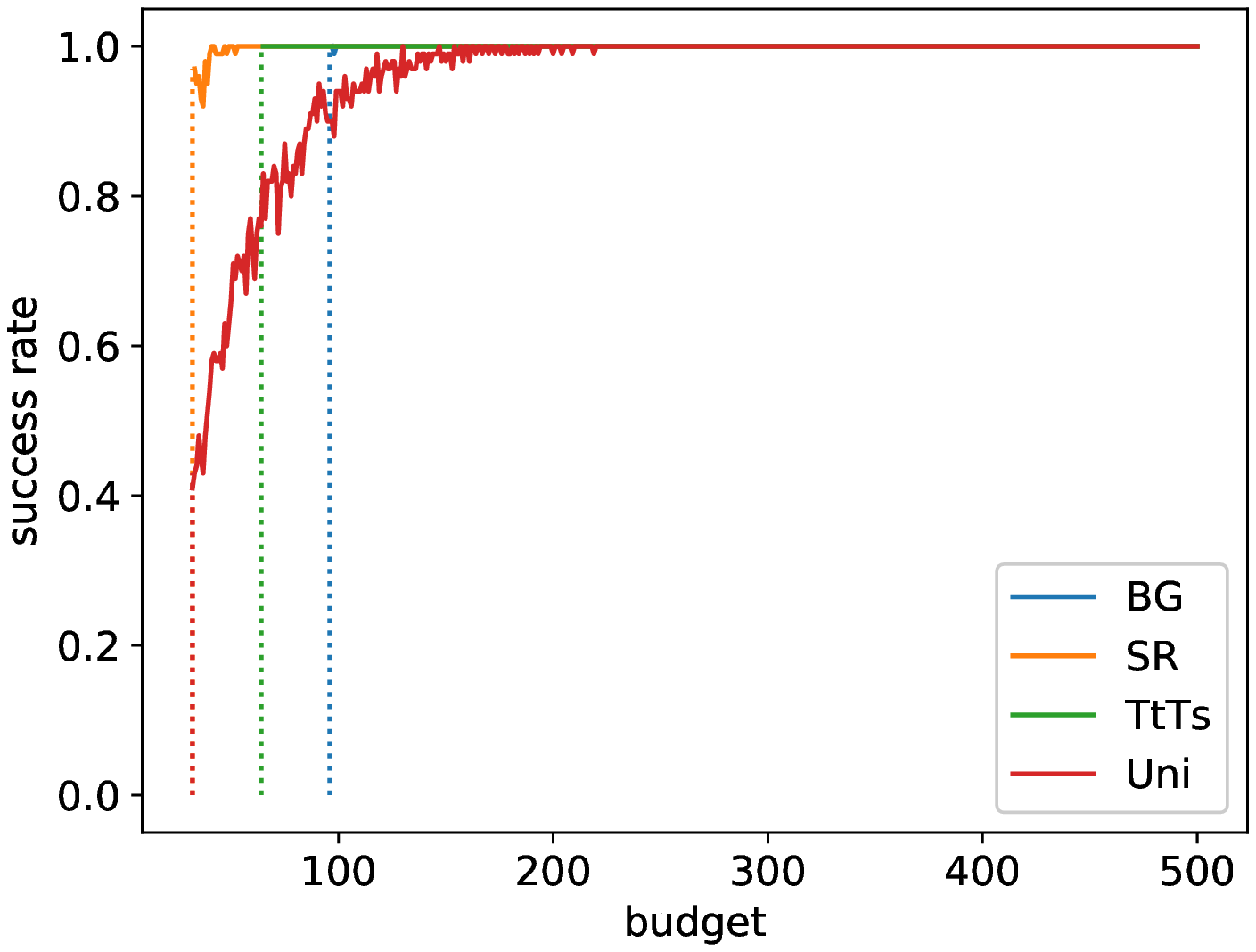} 
    \caption{Bandit run results for $R_{0}=1.8$.}
  \end{subfigure}
  \begin{subfigure}[b]{0.5\linewidth}
    \centering
    \includegraphics[width=0.75\linewidth]{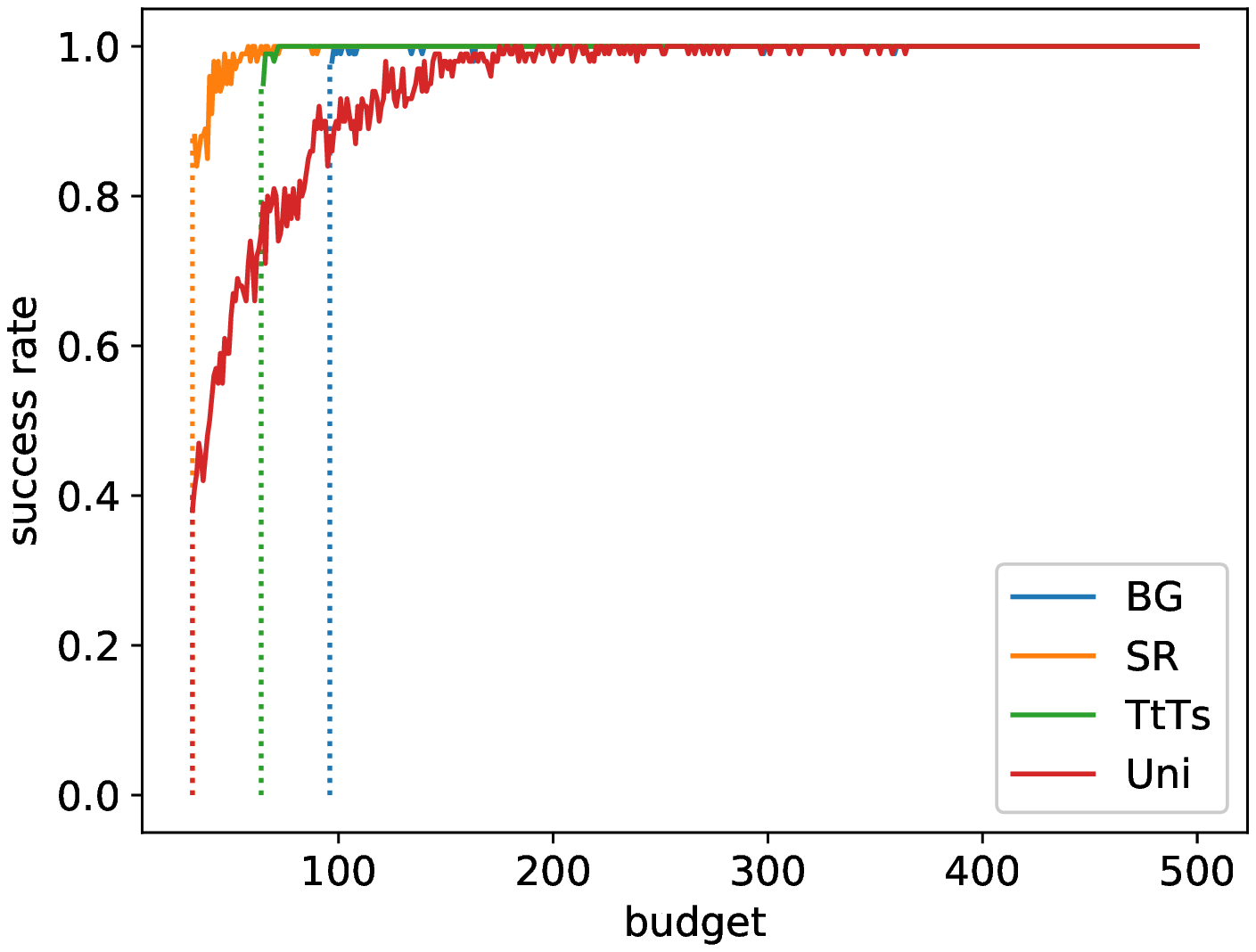}
    \caption{Bandit run results for $R_{0}=2.0$.}
  \end{subfigure}
   \begin{subfigure}[b]{0.5\linewidth}
    \centering
    \includegraphics[width=0.75\linewidth]{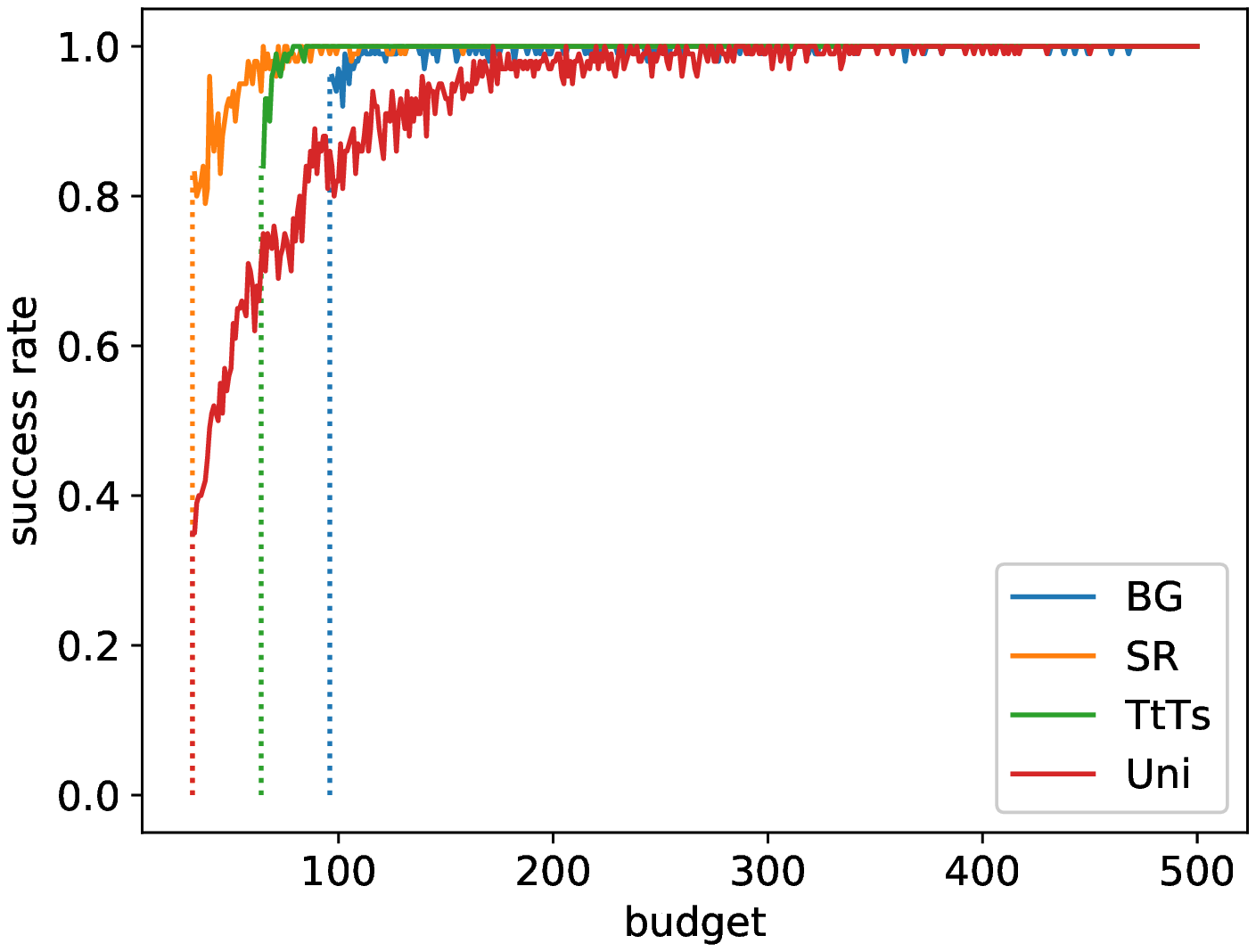}
    \caption{Bandit run results for $R_{0}=2.2$.}
   \end{subfigure}
    \begin{subfigure}[b]{0.5\linewidth}
    \centering
    \includegraphics[width=0.75\linewidth]{bandit_run_24}
    \caption{Bandit run results for $R_{0}=2.4$.}
  \end{subfigure} 
\end{figure}

\newpage
\section{$P_s$ values for Top-two Thompson sampling}
\begin{figure}[!h] 
  \begin{subfigure}[b]{0.5\linewidth}
    \centering
    \includegraphics[width=0.75\linewidth]{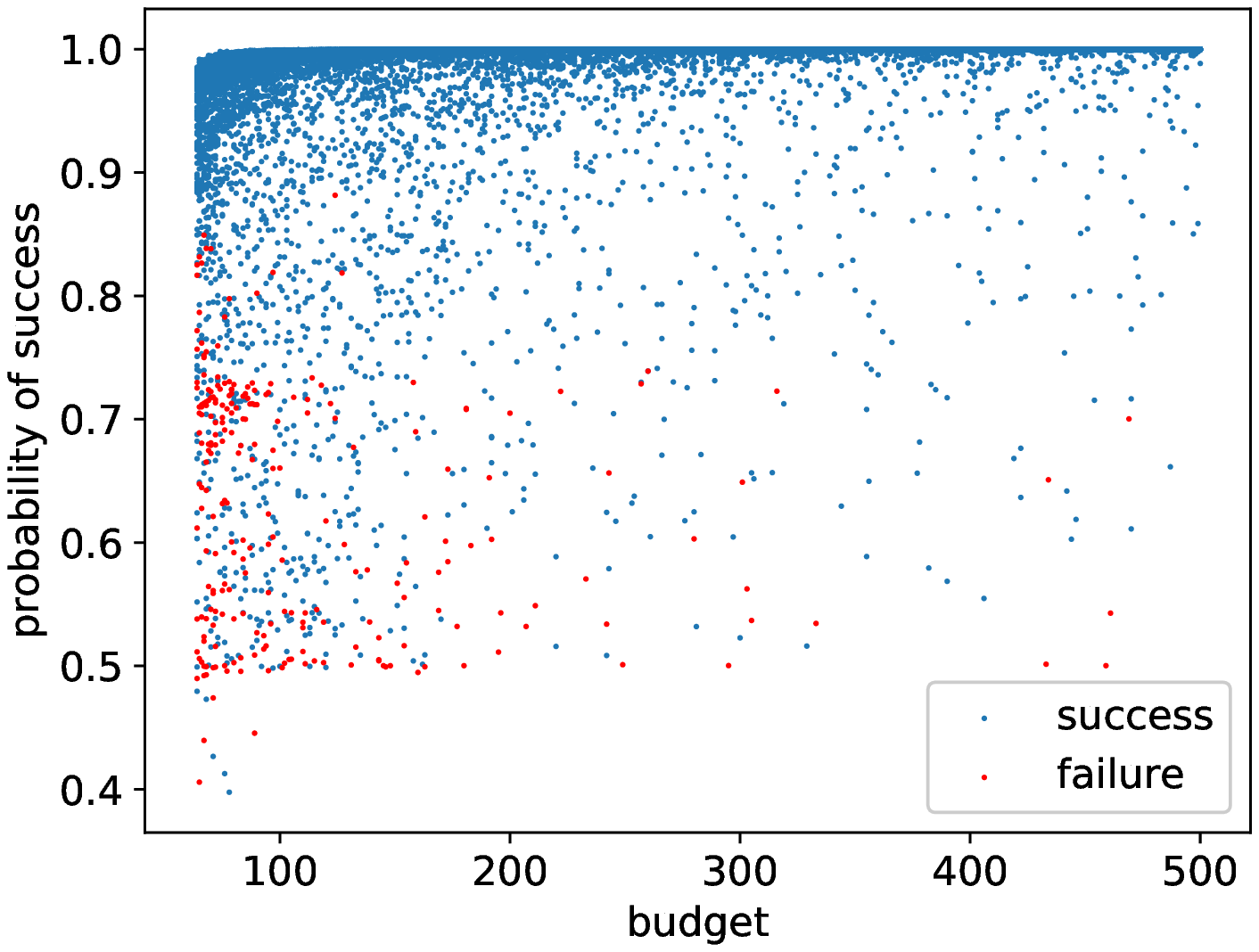} 
    \caption{$P_s$ values for $R_{0}=1.4$.} 
    \vspace{4ex}
  \end{subfigure}
  \begin{subfigure}[b]{0.5\linewidth}
    \centering
    \includegraphics[width=0.75\linewidth]{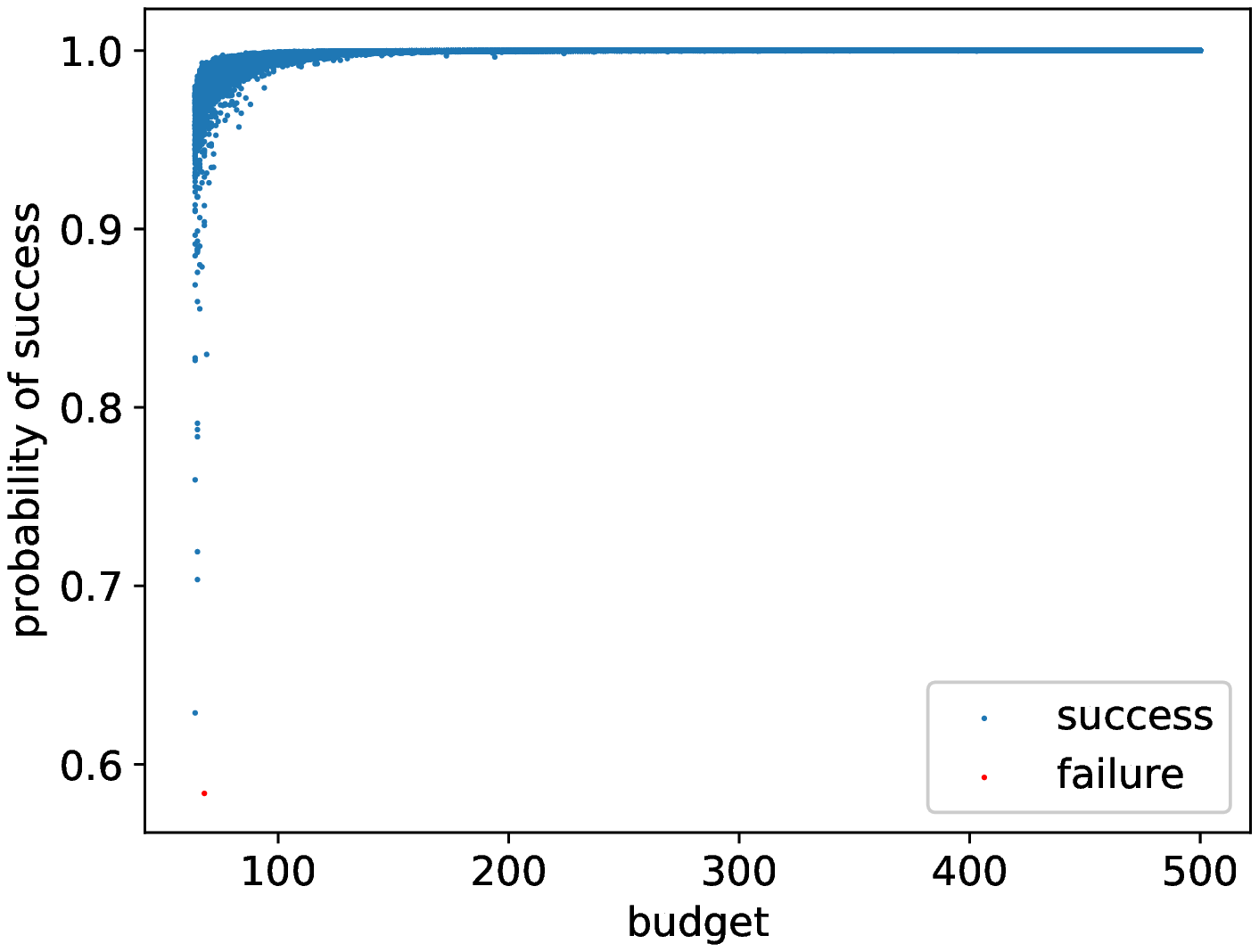} 
   \caption{$P_s$ values for $R_{0}=1.6$.}
    \vspace{4ex}
  \end{subfigure} 
  \begin{subfigure}[b]{0.5\linewidth}
    \centering
    \includegraphics[width=0.75\linewidth]{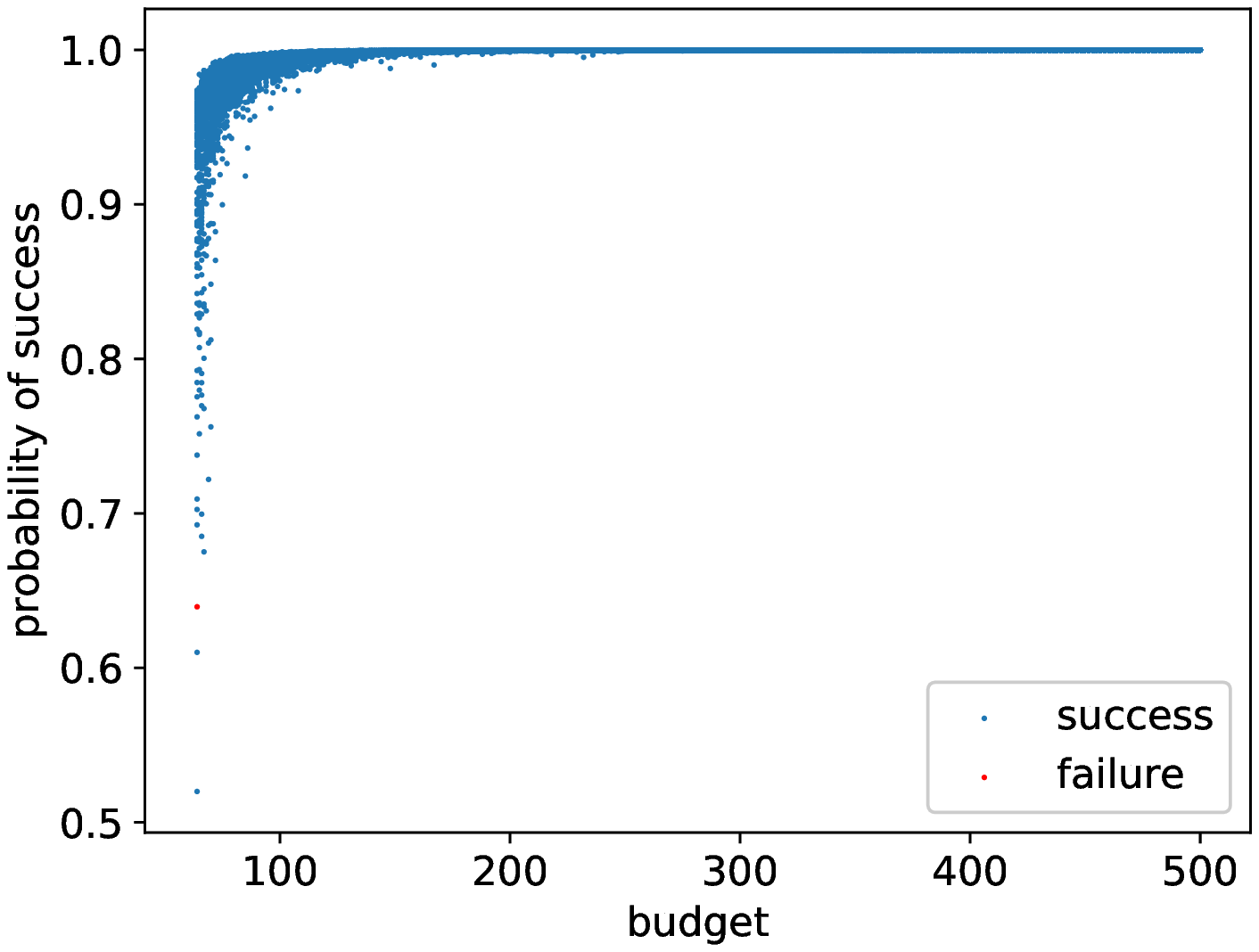} 
    \caption{$P_s$ values for $R_{0}=1.8$.}
  \end{subfigure}
  \begin{subfigure}[b]{0.5\linewidth}
    \centering
    \includegraphics[width=0.75\linewidth]{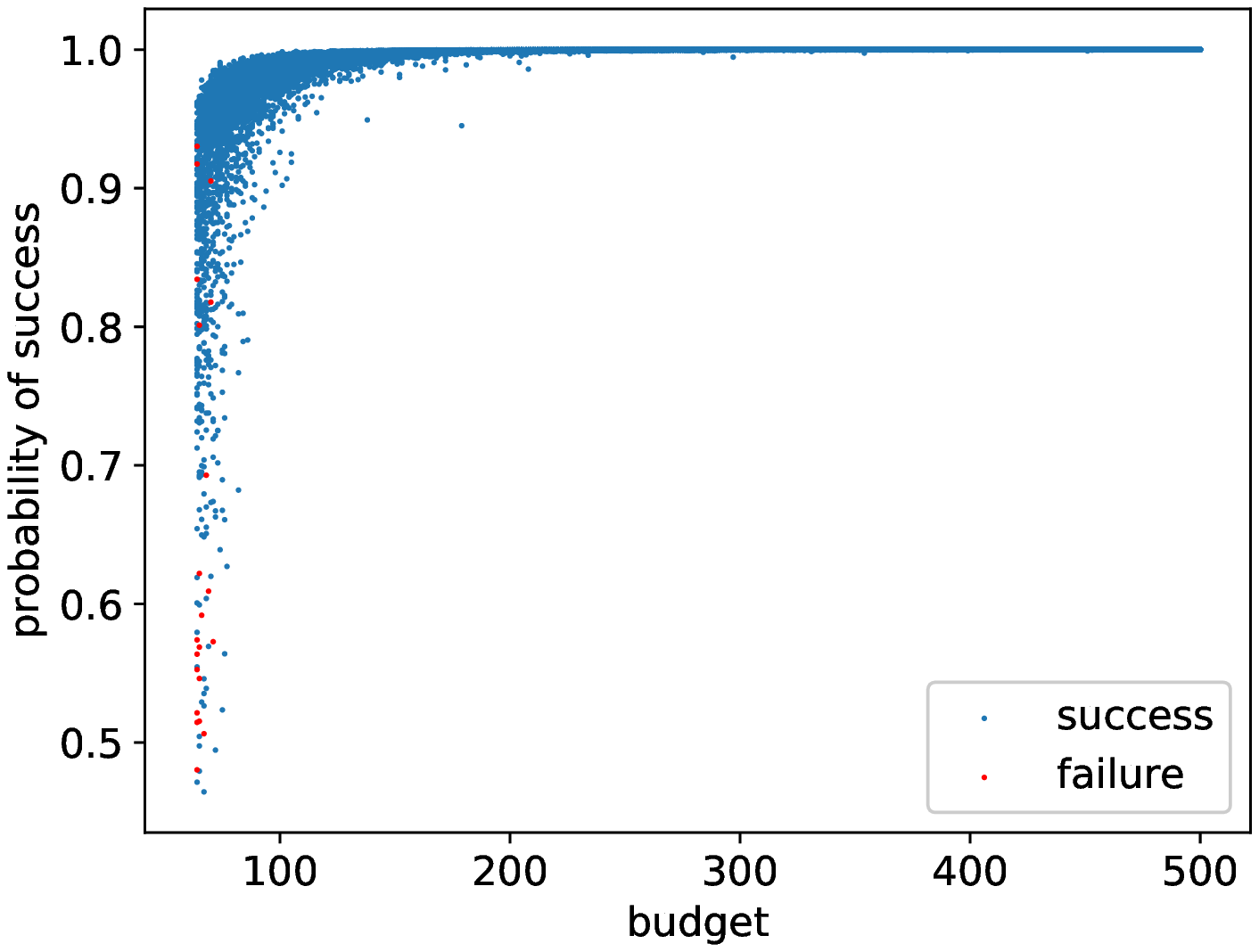}
    \caption{$P_s$ values for $R_{0}=2.0$.}
  \end{subfigure}
   \begin{subfigure}[b]{0.5\linewidth}
    \centering
    \includegraphics[width=0.75\linewidth]{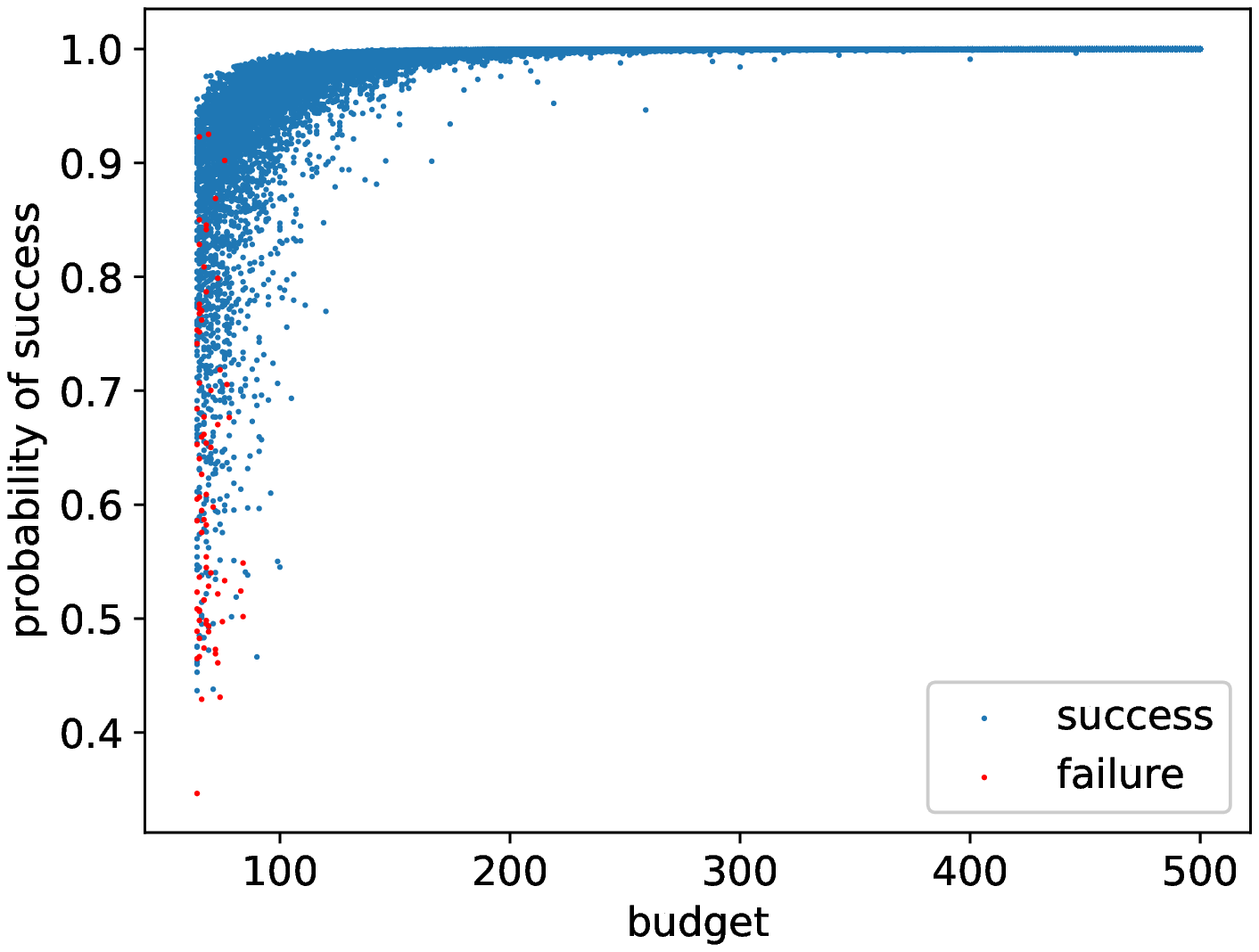}
    \caption{$P_s$ values for $R_{0}=2.2$.}
   \end{subfigure}
    \begin{subfigure}[b]{0.5\linewidth}
    \centering
    \includegraphics[width=0.75\linewidth]{prob_of_success_dist_24}
    \caption{$P_s$ values for $R_{0}=2.4$.}
  \end{subfigure}
\end{figure}

\newpage
\section{Binned distribution of $P_s$ values for Top-two Thompson sampling}
\begin{figure}[!h] 
  \begin{subfigure}[b]{0.5\linewidth}
    \centering
    \includegraphics[width=0.75\linewidth]{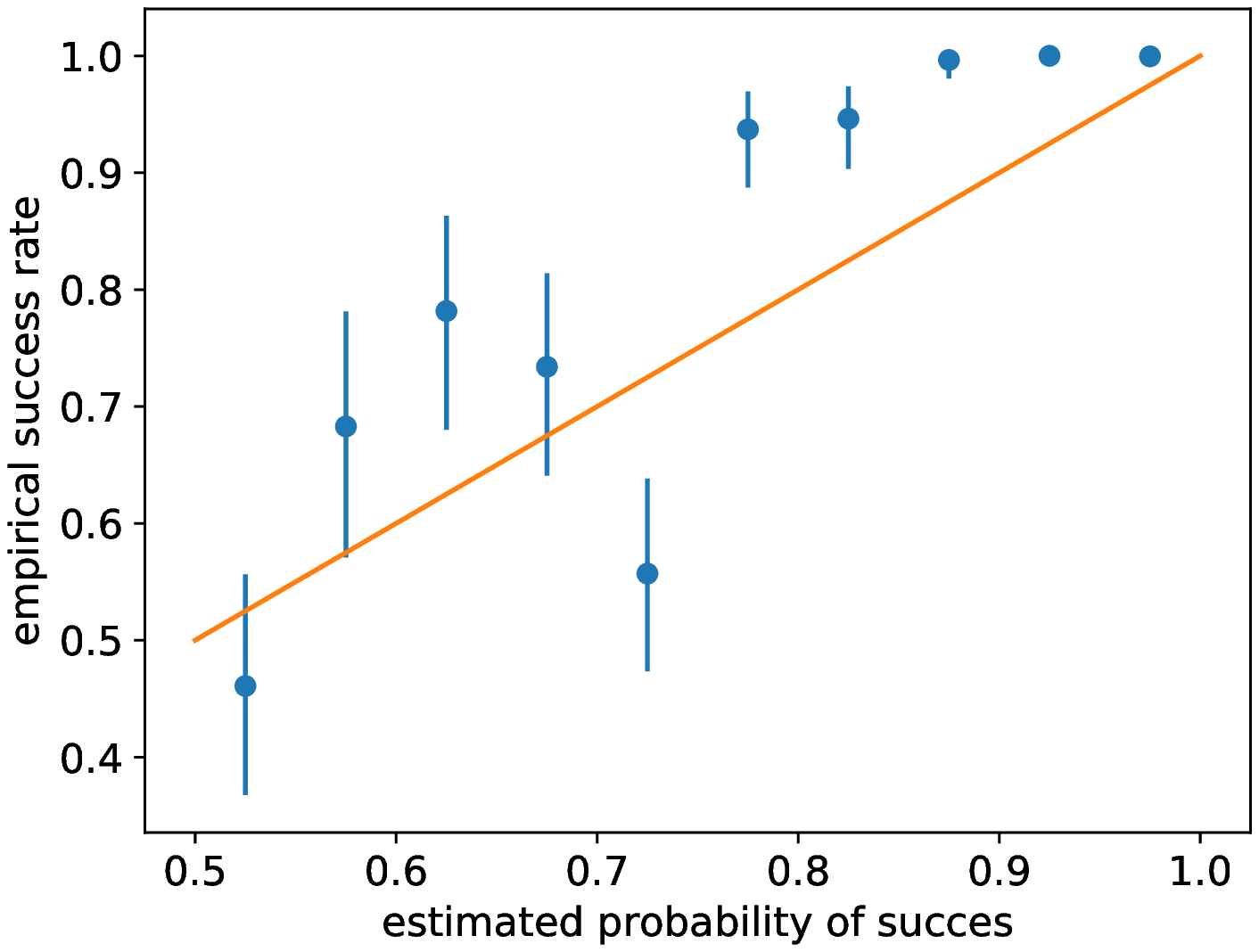} 
    \caption{Binned distribution for $R_{0}=1.4$.} 
    \vspace{4ex}
  \end{subfigure}
  \begin{subfigure}[b]{0.5\linewidth}
    \centering
    \includegraphics[width=0.75\linewidth]{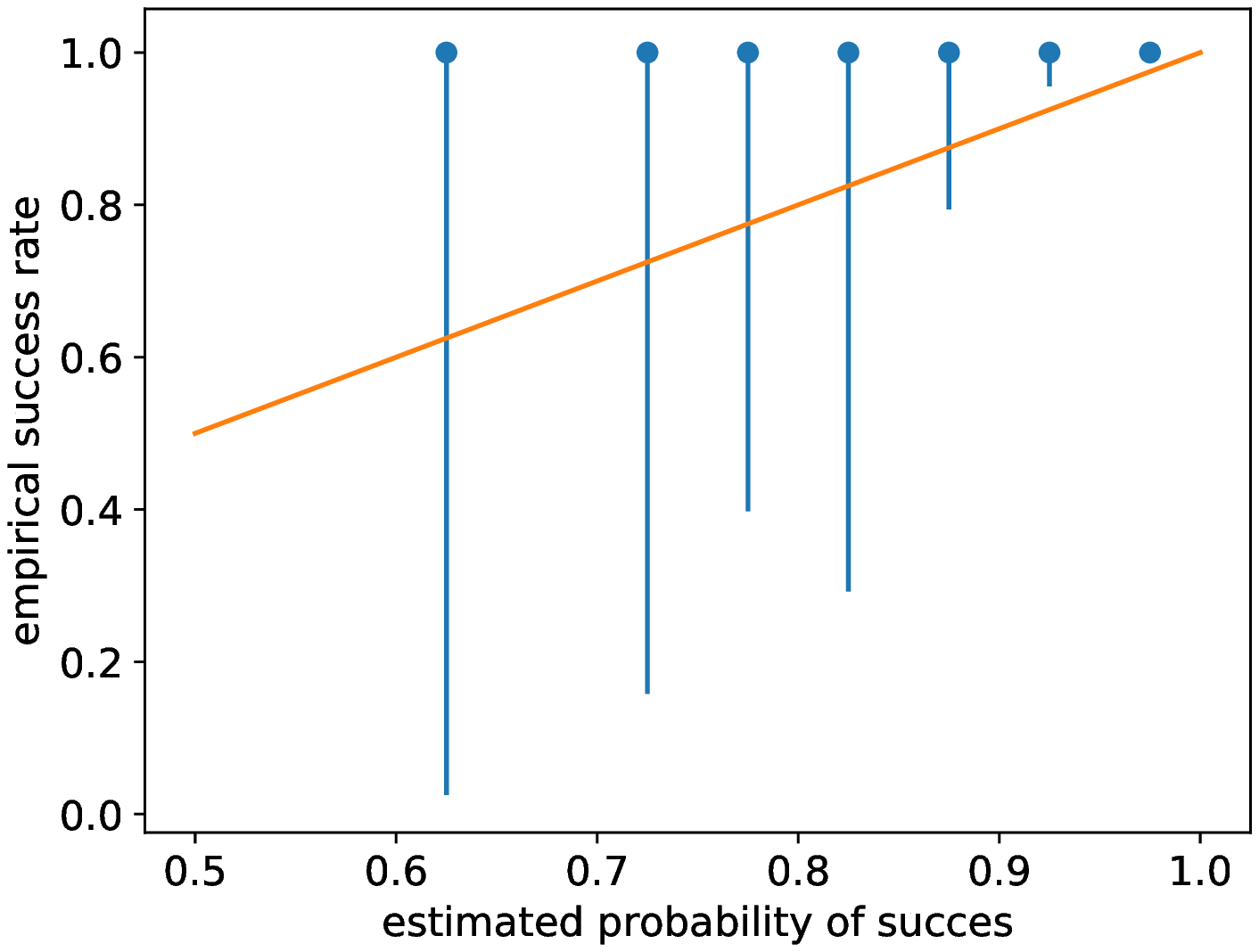} 
   \caption{Binned distribution for $R_{0}=1.6$.}
    \vspace{4ex}
  \end{subfigure} 
  \begin{subfigure}[b]{0.5\linewidth}
    \centering
    \includegraphics[width=0.75\linewidth]{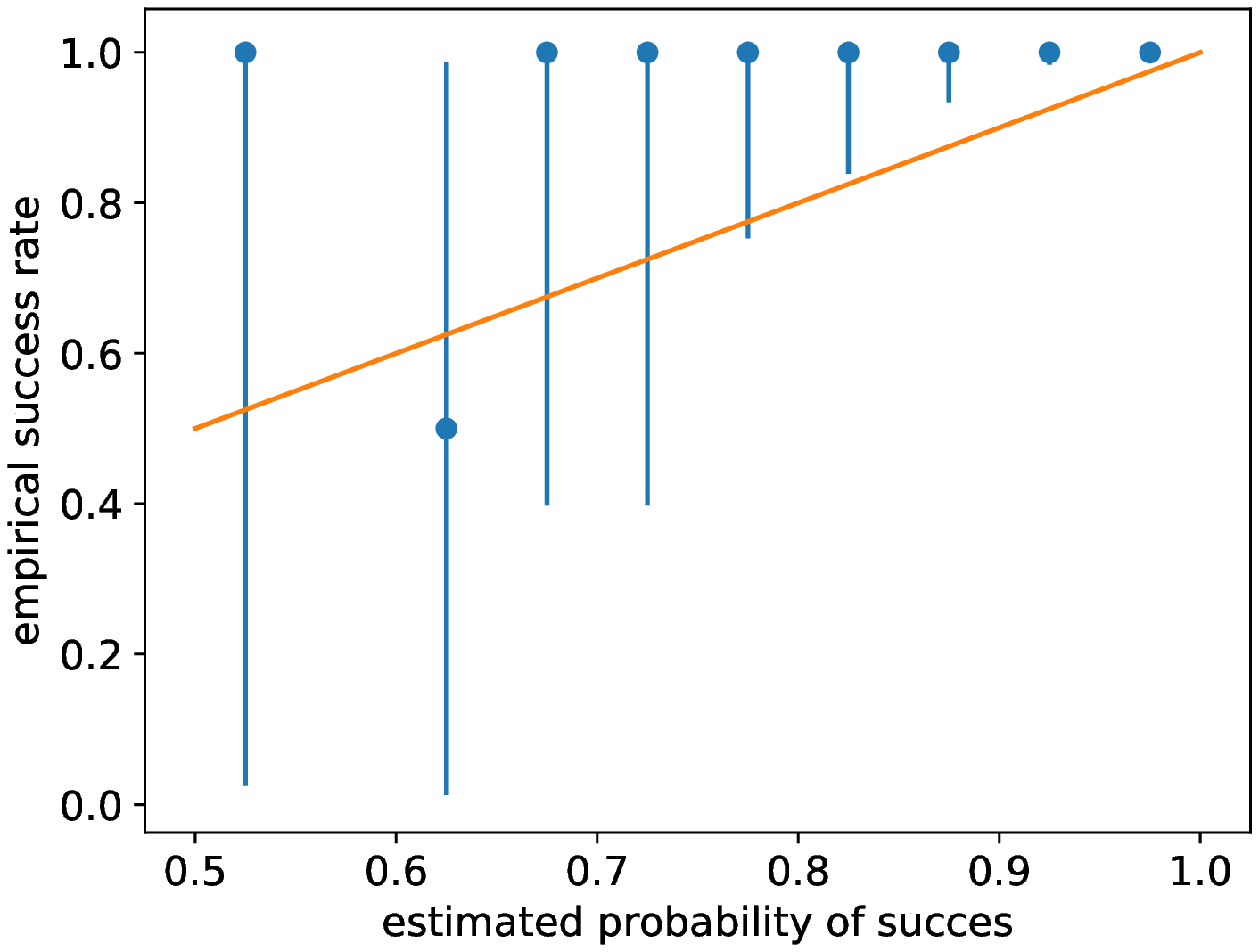} 
    \caption{Binned distribution for $R_{0}=1.8$.}
  \end{subfigure}
  \begin{subfigure}[b]{0.5\linewidth}
    \centering
    \includegraphics[width=0.75\linewidth]{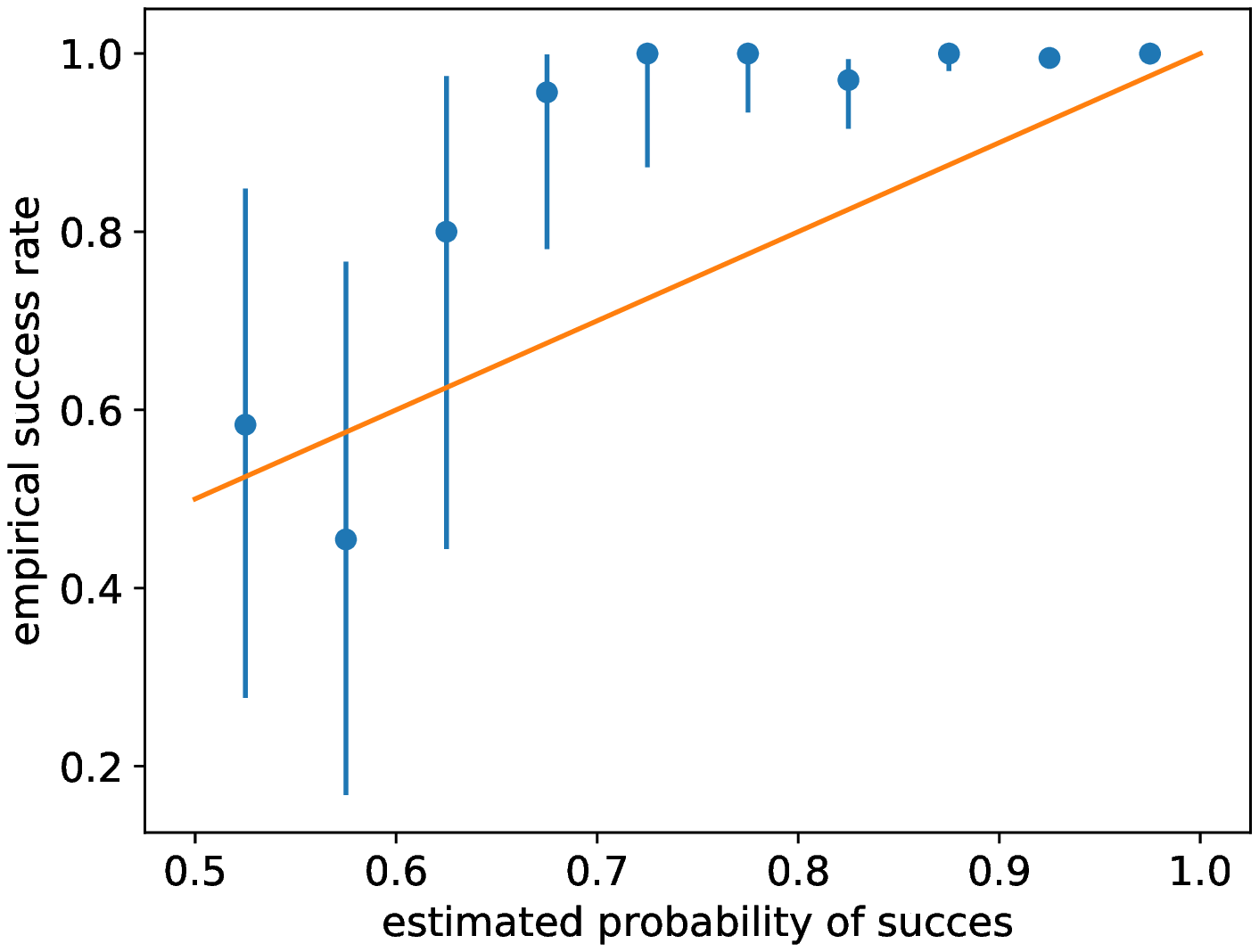}
    \caption{Binned distribution for $R_{0}=2.0$.}
  \end{subfigure}
   \begin{subfigure}[b]{0.5\linewidth}
    \centering
    \includegraphics[width=0.75\linewidth]{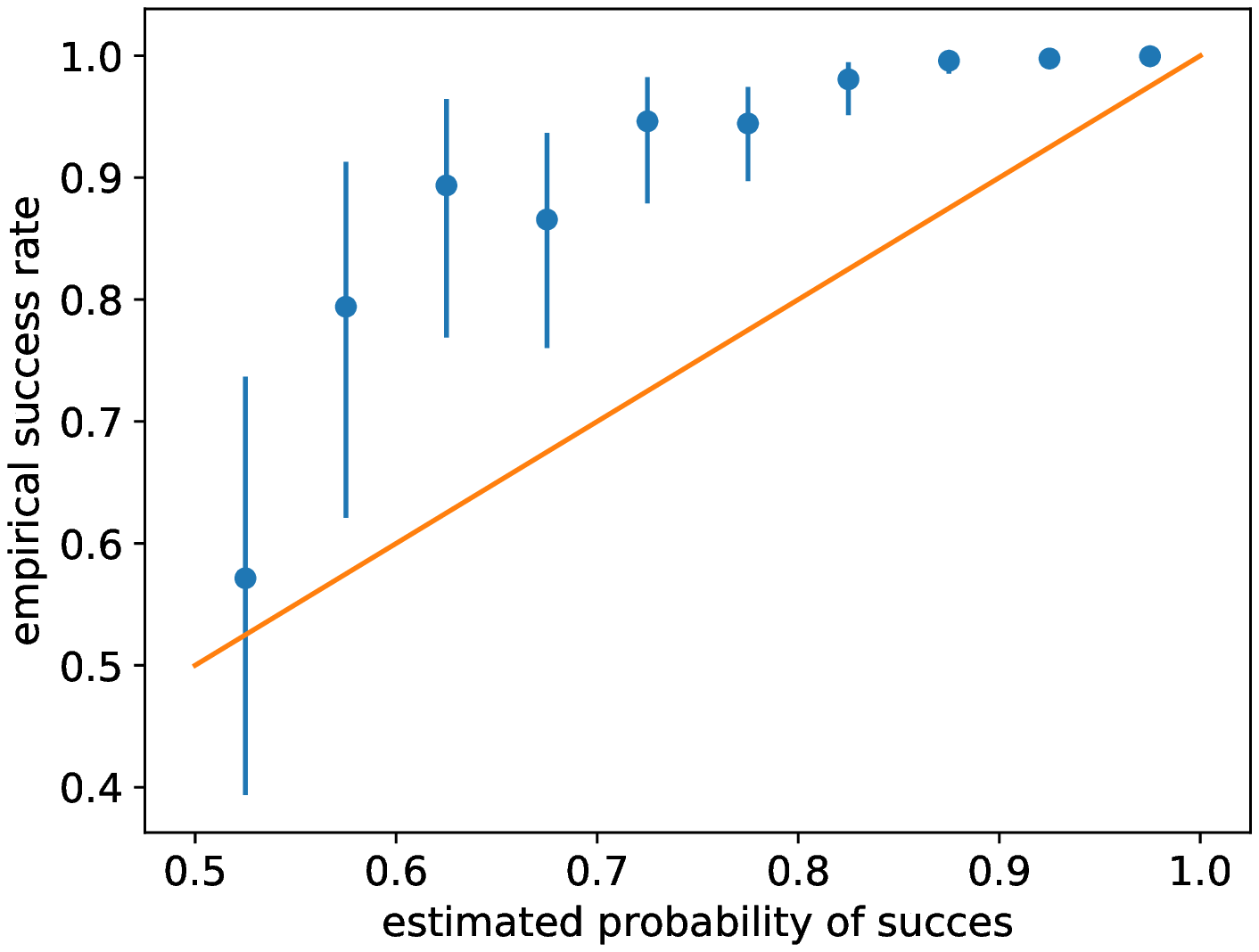}
    \caption{Binned distribution for $R_{0}=2.2$.}
   \end{subfigure}
    \begin{subfigure}[b]{0.5\linewidth}
    \centering
    \includegraphics[width=0.75\linewidth]{binned_prob_of_success_24}
    \caption{Binned distribution for $R_{0}=2.4$.}
  \end{subfigure}
\end{figure}

\newpage
\section{Computational resources}
The simulations were run on a high performance cluster (HPC). On this HPC, we used ``Ivy Bridge'' nodes, more specifically nodes with two 10-core "Ivy Bridge" Xeon E5-2680v2 CPUs (2.8 GHz, 25 MB level 3 cache) and 64 GB of RAM. This infrastructure allowed us to run 20 FluTE simulations per node.

\newpage
\bibliographystyle{plainnat}  
\bibliography{refs_processed}  

\end{document}